\documentclass{article}
\usepackage{bbding}
\usepackage{microtype}
\usepackage{graphicx}
\usepackage{subcaption}
\usepackage{booktabs} 
\usepackage{hyperref}
\usepackage{arydshln}
\usepackage[table]{xcolor}

\usepackage{multirow}
\usepackage{makecell}

\usepackage{multirow}

\usepackage[preprint]{icml2026}

\usepackage{amsmath}
\usepackage{amssymb}
\usepackage{mathtools}
\usepackage{amsthm}
\usepackage{booktabs}
\usepackage{adjustbox}
\usepackage{pifont}
\usepackage{pifont}
\usepackage{amssymb}
\usepackage[capitalize,noabbrev]{cleveref}

\theoremstyle{plain}
\newtheorem{theorem}{Theorem}[section]
\newtheorem{proposition}[theorem]{Proposition}

\theoremstyle{definition}

\theoremstyle{remark}

\newcommand{\cmark}{\ding{51}}
\newcommand{\xmark}{\ding{55}}
\definecolor{DeepSkyBlue4}{RGB}{0,104,139}
\hypersetup{
hidelinks,
colorlinks=true,
urlcolor=DeepSkyBlue4, 
citecolor=cyan,
linkcolor=red
}

\usepackage[textsize=tiny]{todonotes}

\icmltitlerunning{MorphSNN: Adaptive Graph Diffusion and Structural Plasticity for Spiking Neural Networks}

\begin{document}

\twocolumn[
  \icmltitle{MorphSNN: Adaptive Graph Diffusion and Structural Plasticity for Spiking Neural Networks}



  \icmlsetsymbol{equal}{*}

  \begin{icmlauthorlist}
    \icmlauthor{Yongsheng Huang}{neu,gdiist}
    \icmlauthor{Peibo Duan\textsuperscript{\Envelope}}{neu}
    \icmlauthor{Yujie Wu}{polyU}
    \icmlauthor{Kai Sun}{monash}
    \icmlauthor{Zhipeng Liu}{neu}
    \icmlauthor{Jiaxiang Liu}{gdiist}
    \icmlauthor{Guangyu Li}{gdiist}
    \icmlauthor{Changsheng Zhang}{neu}
    \icmlauthor{Bin Zhang}{neu}
    \icmlauthor{Mingkun Xu\textsuperscript{\Envelope}}{gdiist}
  \end{icmlauthorlist}

  \icmlaffiliation{neu}{School of Software, Northeastern University, China}
  \icmlaffiliation{gdiist}{Guangdong Institute of Intelligence Science and Technolog, China}
  \icmlaffiliation{polyU}{The Hong Kong Polytechnic University, China}
  \icmlaffiliation{monash}{Monash University, Australia}

  \icmlcorrespondingauthor{Peibo Duan}{duanpeibo@swc.neu.edu.cn}
  \icmlcorrespondingauthor{Mingkun Xu}{xumingkun@gdiist.cn}

  \icmlkeywords{Machine Learning, ICML}

  \vskip 0.3in
]



\printAffiliationsAndNotice{}  

\begin{abstract}

  Spiking Neural Networks (SNNs) currently face a critical bottleneck: while individual neurons exhibit dynamic biological properties, their macroscopic architectures remain confined within conventional connectivity patterns that are static and hierarchical. This discrepancy between neuron-level dynamics and network-level fixed connectivity eliminates critical brain-like lateral interactions, limiting adaptability in changing environments. To address this, we propose \textbf{MorphSNN}, a backbone framework inspired by biological non-synaptic diffusion and structural plasticity. Specifically, we introduce a \textbf{Graph Diffusion (GD)} mechanism to facilitate efficient undirected signal propagation, complementing the feedforward hierarchy. Furthermore, it incorporates a \textbf{Spatio-Temporal Structural Plasticity (STSP)} mechanism, endowing the network with the capability for instance-specific, dynamic topological reorganization, thereby overcoming the limitations of fixed topologies. Experiments demonstrate that MorphSNN achieves state-of-the-art accuracy on static and neuromorphic datasets; for instance, it reaches \textbf{83.35\%} accuracy on N-Caltech101 with only 5 timesteps. More importantly, its self-evolving topology functions as an intrinsic distribution fingerprint, enabling superior Out-of-Distribution (OOD) detection without auxiliary training.  The code is available at \url{anonymous.4open.science/r/MorphSNN-B0BC}.

\end{abstract}

\section{Introduction}

Recently, spiking neural networks (SNNs) have garnered growing interest due to their biological plausibility and their ability to generate rich temporal dynamics by leveraging the Integrate-and-Fire mechanism of biological neurons~\citep{maass1997networks}. Furthermore, driven by the rapid and mature development of artificial neural networks (ANNs) like Graph Neural Networks (GNNs)~\citep{scarselli2008graph}, ResNet~\citep{he2016deep}, and Transformers~\citep{vaswani2017attention}, SNNs have advanced substantially, offering energy-efficient, real-time intelligent computation~\citep{roy2019towards}. 
 
As SNNs inherit ANN-based techniques, many of the challenges inherent to ANNs have also been carried over into SNNs. A particularly critical issue stems from the static hierarchical architecture of ANNs, which constrains the evolutionary flexibility of SNNs and, in particular, hampers their ability to effectively encode dynamically changing neuromorphic data streams in two key ways. From the perspective of ``static", the architecture of the model is frozen post-training, making it difficult to adapt to complex and volatile real-world scenarios. As a result, a plethora of post-hoc remedial methods have emerged for new task learning, such as Continual Learning and Out-of-Distribution (OOD) detection~\citep{xu2024adaptive, avramovic2025out}. Existing efforts toward dynamic SNNs, such as pruning~\citep{han2023adaptive, han2023enhancing, han2025adaptive} and Neural Architecture Search (NAS)~\citep{ kim2022neural, pan2024brain}, concentrate on discovering efficient architectures by eliminating redundant connections. However, these modifications are typically restricted to the training-loop level. After deployment, the network structure becomes fixed, failing to achieve real-time structural adjustments at the inference timestep granularity in response to specific instances.

From the ``hierarchical” perspective, the unidirectional vertical stacking of layers inevitably leads to network degradation and increased signal latency~\citep{fang2021deep}. Although residual connections~\citep{hu2024advancing} are frequently employed to alleviate these effects, the lack of adaptable horizontal interactions gives rise to the feature binding problem~\citep{greff2020binding, zheng2022dance}. These limitations, already present in ANNs, become even more pronounced in SNNs due to their reliance on binary spike-based communication. While the self-attention mechanism ~\citep{lee2025spiking} has been proposed as an effective solution, it is substantially more powerful for selective fusion than for global information diffusion~\citep{dong2025can}. GNN-based SNNs~\citep{xu2023exploiting} provide an alternative that does not adhere to a strict hierarchical structure, but they typically depend on inflexible image-to-graph conversions, which disrupt the spatial coherence of neuromorphic streams.

Fundamentally, biological nervous systems operate on a different principle, where countless neurons are organized in a complex graph structure and exhibit efficient and precise computation through \textbf{non-synaptic diffusion} and \textbf{structural plasticity}~\citep{d2016modeling}. Specifically, non-synaptic diffusion enables excitatory signals to spread in all directions across an undirected physical substrate via volume transmission~\citep{mohan2011molecular}. Meanwhile, structural plasticity enables the millisecond-scale dynamic reorganization of synapses in response to input stimuli~\citep{lamprecht2004structural}. The synergistic interaction of these two mechanisms empowers neuronal populations to transcend static hierarchical constraints and rapidly synchronize states on a global scale~\citep{singer1999neuronal}.

Inspired by the analysis above, this paper presents a novel attempt to develop a model that transcends the conventional static hierarchical architectures used in SNNs. To this end, we propose a graph-structured network, \textbf{MorphSNN}, which abstracts the principles of the two mechanisms discussed earlier and thus tackles objective classification tasks in a more fundamental yet efficient way. Concretely, MorphSNN consists of two core components. The first is \textbf{Graph Diffusion (GD)}, a mathematically principled mechanism that characterizes signal propagation as a continuous, energy-minimizing diffusion process, thereby capturing undirected propagation in non-synaptic diffusion. The second component, \textbf{Spatio-Temporal Structural Plasticity (STSP)}, combines graph attention with plasticity modulation to emulate structural plasticity processes. The main contribution of this work is summarized as follows:



\begin{itemize}
    \item \textbf{Framework}: In contrast to GNN-based SNNs, we embed graph dynamics directly into convolutional feature maps, improving global information propagation while maintaining native spatial organization. 
    \item \textbf{Model}: The GD supports undirected information exchange to efficiently capture global context. Unlike NAS, the STSP allows the network to reconfigure its topology at each inference timestep, yielding adaptive, instance-specific structures for representing complex neuromorphic inputs.

    \item \textbf{Analysis}: A theoretical analysis shows that GD ensures dynamic stability, allowing the model to learn with a monotonic convergence toward the energy equilibrium, which in turn helps prevent over-activation and erroneous feature binding.
    \item \textbf{Implementation}: Experiments show that MorphSNN delivers SOTA results on neuromorphic benchmarks and remains competitive on conventional static datasets. Crucially, for new task learning that involves OOD detection, its dynamic topology serves as a training-free intrinsic metric, achieving performance on par with memory-heavy baseline methods.

\end{itemize}

\section{Related Work}

\subsection{Typical Hierarchical Structure based SNNs}
The architecture evolution of SNNs has largely followed that of ANNs, focusing on deepening fixed hierarchical structures. ~\citet{matsugu2002convolutional} established a simple three-layer convolutional SNN, which was later scaled to deep architectures using ResNet-based backbones~\citep{he2016deep}. Specifically, Spiking ResNet~\citep{hu2021spiking} pioneered the conversion of pre-trained ResNet to the spiking domain, extending SNN depth to 18 layers. SEW-ResNet~\citep{fang2021deep} modified the residual block by adding identity mappings to post-activation residuals, enabling 150+ layer SNNs.  Recently, the success of Vision Transformers~\citep{dosovitskiy2020image} has inspired spiking transformers. For instance, Spikformer~\citep{zhou2022spikformer} was the first work to integrate self-attention into SNNs, achieving 74\% accuracy on ImageNet-1k. Although self-attention theoretically constructs a global interaction graph, these models remain bound by a layer-wise, feed-forward execution flow. They rely on stacking rigid layers to expand the receptive field rather than utilizing continuous signal diffusion.  Furthermore, \citet{2025_huang_cognisnn} introduced the random graph architecture to construct SNNs, demonstrating the feasibility and significant potential of graph-structured topologies in the spiking domain.

\subsection{Static Paradigms in Conventional SNNs}
 
While SNNs introduce temporal dynamics via spiking neurons, their underlying topological structures predominantly follow a  static paradigm, lacking the adaptability to complex, varying environments during inference.  Early Liquid State Machines~\citep{maass2011liquid} utilized high-dimensional reservoir dynamics but relied on fixed, task-agnostic topologies. These task-independent connections remain constant throughout both training and inference, contrasting with the rapid synaptic rewiring observed in biological circuits.  To introduce dynamics, PLIF~\citep{fang2021incorporating} and ILIF~\citep{sun2025ilif} provided dynamic and time-varying modulation of spiking neurons. However, these approaches operate primarily at the neuron level, regulating neuronal firing rates without physically reshaping the information pathways. For continual learning, SA-SNN~\citep{shen2024efficient} advanced this by employing trace-based K-WTA and variable thresholds to achieve selective activation. DSD-SNN~\citep{han2023enhancing} proposed dynamic pruning and growth of kernels in the training loop. Similarly, recent TD-MCL~\citep{han2025continual} mimicked brain development by employing evolutionary strategies to grow long-range pathways while pruning local redundancy. While effective for knowledge transfer, such developmental plasticity stabilizes after training, leaving the network topology static during the inference of specific instances. CogniSNN~\citep{huang2025cognisnn} preliminarily explored inference-time dynamics based on pre-defined graph topological substructures, demonstrating the potential of such dynamics in enhancing model robustness.


\section{Preliminaries}

\subsection{LIF Model}

We employ the Leaky Integrate-and-Fire (LIF) model to process discrete input. Unlike the continuous activations in ANNs, LIF neurons are characterized by their rich dynamics. The membrane potential $u^{(t)}_i$ of neuron $i$ at timestep $t$ is determined by the decay factor $\tau_{\text{decay}}$ and the input current $C_i^{(t)}$ after synaptic integration:
\begin{align}
     u^{(t)}_i = \tau_{\text{decay}} \cdot u^{(t-1)}_i + C^{(t)}_i - V_{\text{th}}\cdot s_i^{(t-1)}.
    \label{eq1:mebrane_potential}
\end{align}%
When $u^{(t)}_i$ accumulates beyond the threshold $V_{\text{th}}$, the neuron fires a spike and resets its potential. The spike generation process is defined by the Heaviside step function $\Theta(\cdot)$ as:
\begin{align}
     {s}_i^{(t)} = \Theta(u_i^{(t)}- V_{\text{th}}) \in \left \{  0,1 \right \}.
    \label{eq2:firing}
\end{align}%
This event-driven nature grants SNNs energy efficiency while introducing non-differentiability. Thus, we employ surrogate gradient~\citep{2018_wu_spatio} for backpropagation.

\subsection{GCN and GAT}
Graph Convolutional Networks (GCNs;~\citep{kipf2016semi}) enable nodes to exchange information via neighborhood aggregation. Based on spectral graph theory, standard GCNs approximate graph convolution using first-order Chebyshev polynomials. The layer propagation rule is defined as:
\begin{align}
     \mathbf{H}’ = \sigma \left ( \tilde{\mathbf{D}}^{-\frac{1}{2} } \tilde{\mathbf{A}} \tilde{\mathbf{D}}^{-\frac{1}{2} } \mathbf{H} \mathbf{W}\right ),
    \label{eq3:gcn}
\end{align}%
\noindent where $\tilde{\mathbf{A}} = \mathbf{A} + \mathbf{Id}$ is the adjacency matrix with self-loops ($\mathbf{Id}$ denotes the identity matrix), and $\tilde{\mathbf{D}}$ is the degree matrix. $\mathbf{W}$ is a learnable weight matrix, while $\mathbf{H}$ and $\mathbf{H}'$ represent the node features before and after graph convolution, respectively. $\sigma(\cdot)$ denotes the activation function.

To capture the varying importance of neighbors, Graph Attention Networks (GAT;~\citep{velivckovic2017graph}) introduce an attention mechanism. The attention coefficient $e_{ij}$ between nodes $i$ and $j$ is computed as follows:
\begin{align}
     e_{ij} = \text{LeakyReLU}\left ( \mathbf{a}^\top [\mathbf{W}\mathbf{h}_i \mathbin{\|} \mathbf{W}\mathbf{h}_j] \right),
    \label{eq4:gat}
\end{align}%
\noindent where $\mathbf{h}_i$ and $\mathbf{h}_j$ are input features of nodes $i$ and $j$, $\mathbf{a}$ is the learnable attention vector, and $\mathbin{\|}$ indicates the concatenation.

\section{Methodology}

\begin{figure*}[t]
    \centering
    \includegraphics[width=\textwidth]{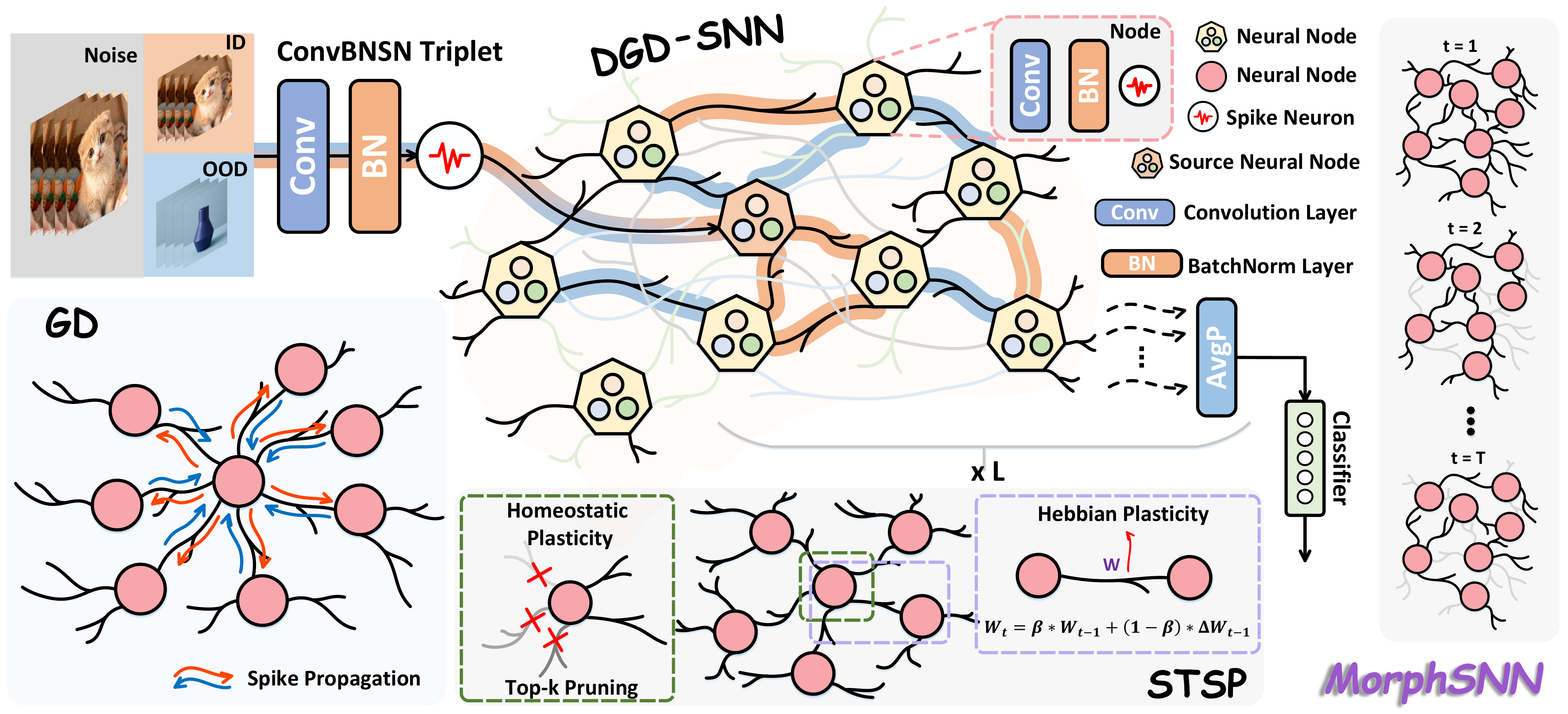}
    \caption{\textbf{The overall framework of MorphSNN.} Input spikes flow from the direct encoding layer through $L$ stacked DGD-SNN layers to the final classifier. \textbf{Top Left:} The model exhibits robustness against noise, with orange (In-Distribution, ID) and blue (OOD) samples activating distinct graph structures for OOD detection. \textbf{Top Center:} Each node is implemented as a \text{ConvBNSN} triplet. \textbf{Bottom Left:} The GD executes undirected signal propagation using the Laplacian operator. \textbf{Bottom Center:} The STSP strengthens synapses via Hebbian plasticity while maintaining stability via homeostatic plasticity, driving the temporal evolution of the entire graph structure (\textbf{Right}).}
    \label{fig:Framework}
\end{figure*}

\subsection{Overview of MorphSNN}

As illustrated in Fig.~\ref{fig:Framework}, MorphSNN integrates a direct encoding layer, multiple Dynamic Graph Diffusion SNN (DGD-SNN) layers, and a linear classifier. Serving as the core feature extractor, each DGD-SNN layer is mathematically modeled as a graph $\mathcal{G} = (\mathcal{V}, \mathcal{E})$ with $N = |\mathcal{V}|$ nodes. Each node $v_i \in \mathcal{V}$ is instantiated as a fundamental \textbf{ConvBNSN} triplet used by ~\citet{huang2025cognisnn}: $\text{ConvBNSN}(x)=\text{SN}(\text{BN}(\text{Conv}(x)))$. The graph topology is parameterized by an adjacency matrix $\mathbf{S} \in \mathbb{R}^{N \times N}$, where the entry $\mathbf{S}_{ij}$ denotes the synaptic weight from node $v_j$ to $v_i$ (with $\mathbf{S}_{ij} \neq 0$ implying $(v_j, v_i) \in \mathcal{E}$). In each batch, $\mathbf{S}$ is initialized as an all-ones symmetric matrix to ensure initial topological equivalence.


Without loss of generality, we designate the first node $v_0$ as the dedicated source node for receiving the output from the direct encoding layer. This strategy establishes a stable spatial reference frame for the diffusion process, avoiding the optimization ambiguity associated with random source selection. For a given node $v_i$, we denote its input and output as $\mathbf{I}_i$ and $\mathbf{O}_i$, respectively. Consequently, for the source node $v_0$ with input $\mathbf{I}^{(t)}_0$, its output $\mathbf{O}^{(t)}_0$ at timestep $t$ is expressed as:
\begin{align}
    \mathbf{O}_0^{(t)} = \text{AvgP}(\text{ConvBNSN}_0(\mathbf{I}_0^{(t)})),
    \label{eq5:source_node}
\end{align}%
\noindent where $\text{AvgP}(\cdot)$ denotes a $2 \times 2$ average pooling operation. 

To overcome the limitations of static connectivity, we use \textbf{STSP} (\textbf{Section \ref{subsec: STSP}}) to update the matrix $\hat{\mathbf{S}}^{(t)}$ at timestep $t$:
\begin{align}
    \hat{\mathbf{S}}^{(t)} = \text{STSP}(\mathbf{S}^{(t-1)}, \left \{ \mathbf{O}_0^{(t)},\mathbf{O}_1^{(t-1)}, \dots, \mathbf{O}_{N-1}^{(t-1)}\right \}).
    \label{eq6:stsp}
\end{align}%
Diverging from the rigid hierarchical transmission of standard feedforward networks, we employ \textbf{GD} (\textbf{Section \ref{subsec: GD}}) for undirected signal propagation among nodes:
\begin{align}
    \left \{ \mathbf{I}_1^{(t)},\dots,  \mathbf{I}_{N-1}^{(t)}\right \}=\text{GD}(\mathbf{O}_0^{(t)},\hat{\mathbf{S}}^{(t)}).
    \label{eq7:gd}
\end{align}%
Subsequently, we utilize the $\text{ConvBNSN}$ triplet within each node to extract features:
\begin{align}
    \mathbf{O}^{(t)}_i = \text{ConvBNSN}_i(\mathbf{I}_i^{(t)} ), \quad 1 \le i< N.
    \label{eq:feature_extract_new}
\end{align}%
Finally, the output $\mathbf{O}^{(t)}$ of the DGD-SNN is derived via a learnable weighted readout mechanism:
\begin{align}
    \mathbf{O}^{(t)} = \text{AvgP}(\sum_{i=0}^{N-1}\sigma_{1}(w_i)\cdot \mathbf{O}^{(t)}_i), \quad 0\le i< N.
    \label{eq9:total_output}
\end{align}%
Here, $w_i$ denotes the weight assigned to each node $v_i$, and $\sigma_{1}(\cdot)$ is the Sigmoid function. For stacked layers, the aggregated output $\mathbf{O}^{(t)}$ of the $l$-th layer serves directly as the input for the source node of the $(l+1)$-th layer.

\subsection{Spatio-Temporal Structure Plasticity}
\label{subsec: STSP}

To infer robust, instance-specific connectivity from dynamic neural activities, we propose \textbf{STSP}. Our design draws inspiration from biological structural plasticity, synthesizing four key principles: \textbf{Hebbian plasticity} to reinforce co-activated connections based on firing rates, \textbf{graph attention} to model inter-nodal correlations, \textbf{homeostatic plasticity} to stabilize synaptic weights, and \textbf{Short-Term Plasticity (STP)} to regulate weight updates along the temporal dimension.

For each input batch, the adjacency matrix is initialized as an all-ones matrix $\mathbf{S}^{(0)}=\mathbf{1}_{N\times N}$, representing a uniform prior. Inferring the dynamic topology at timestep $t$ requires the instantaneous states of all nodes. However, this presents a causal constraint: the source node $v_0$ receives immediate input, whereas the states of non-source nodes $v_{i>0}$ are updated only after the subsequent graph diffusion process.

To address this, we construct a causal hybrid state tensor $\tilde{\mathbf{O}}_i^{(t)} \in \mathbb{R}^{C\times H \times W} $ that combines the current source input with the historical state of other nodes:
\begin{align}
    \tilde{\mathbf{O}}_{i}^{(t)}=\left\{\begin{array}{lll}
\mathbf{O}_{0}^{(t)}, &i=0  \text { (Source node) } \\
\mathbf{O}_{i}^{(t-1)},&  0<i<N \text { (Others) }
\end{array}\right..
    \label{eq:hybrid_state}
\end{align}
At $t=1$, states $\mathbf{O}_{1:N-1}^{(t-1)}$ are initialized as zero tensor.

To project this hybrid state into a compact semantic embedding guided by the Hebbian principle (fire together, wire together), we employ global average pooling $\text{GAP}(\cdot)$ to aggregate spatial spikes into a channel-wise firing rate vector representing the global firing rate of each channel. The whole operation follows:
\begin{align}
    \mathbf{h}_i^{(t)} = \sigma_2\left (  \mathbf{W}_{\text{proj}} \cdot \text{GAP}(\tilde{\mathbf{O}}_i^{(t)}) \right ) , \quad 0\le i <N,
    \label{eq10:project}
\end{align}%
\noindent where $\mathbf{W}_{\text{proj}}$ denotes a learnable projection matrix, and $\sigma_2$ is the ReLU activation function. The resulting $ \mathbf{h}_i^{(t)}$ serves as the instantaneous feature vector for node $v_i$ used in subsequent topology inference.

To incorporate historical context into topology inference, we introduce a synaptic trace $\mathbf{Tr}_i^{(t)}$ as a smoothed surrogate for node activity, calculated as follows:
\begin{align}
    \mathbf{Tr}_i^{(t)} = \lambda\mathbf{Tr}_i^{(t-1)} + (1-\lambda)\mathbf{h}_i^{(t)}, \quad 0\le i <N.
    \label{eq11:trace}
\end{align}%
Here, $\lambda \in [0,1]$ serves as a hyperparameter controlling history retention. This temporal smoothing filters out transient noise, ensuring that the structural evolution is driven by stable, long-term semantic associations.

Furthermore, to capture diverse correlation patterns, we project the trace $\mathbf{Tr}_i^{(t)}$ into multiple semantic subspaces using a multi-head mechanism. For the $m$-th head ($m = 1,...,M_{\text{head}}$), the projected feature $\mathbf{h}_{i,m}^{(t)}$ is obtained by
\begin{align}
    \mathbf{h}^{(t)}_{i,m} = (\mathbf{W}_{\text{head}} \cdot \mathbf{Tr}_i^{(t)} )_m, \quad 0\le i <N,
    \label{eq12:multi_head}
\end{align}%
\noindent where $\mathbf{W}_{\text{head}}$ is a learnable transformation matrix, and $(\cdot)_m$ denotes vector slicing operation. We set $M_{\text{head}}=4$.

Next, we compute pairwise correlation scores between nodes $v_i$ and $v_j$ ($0\le i,j <N$) in the $m$-th subspace. Unlike standard GAT, we compute dense bidirectional scores $e_{ij,m}^{(t)}$ to facilitate subsequent global topology inference:
\begin{align}
    e_{ij,m}^{(t)} = \text{LeakyReLU}\left (\mathbf{a}_m^\top [\mathbf{h}^{(t)}_{i,m} \mathbin{\|} \mathbf{h}^{(t)}_{j,m}]\right ),
    \label{eq13:e}
\end{align}%
\noindent where $\mathbf{a}_m$ and $\mathbin{\|}$ denote the learnable weight vector for the $m$-th attention head and concatenation, respectively.

To enforce topological symmetry and fuse information across subspaces, we average the bidirectional scores within each head and subsequently aggregate across all $M_{\text{head}}$ heads. The symmetric attention score $\bar{e}_{i j}^{(t)}$ is derived as:
\begin{align}
   \bar{e}_{i j}^{(t)}=\frac{1}{M_{\text{head}}} \sum_{m=1}^{M_{\text{head}}} \left(\frac{e_{i j, m}^{(t)}+e_{j i, m}^{(t)}}{2}\right).
    \label{eq14:Symmetrization}
\end{align}%
To maintain homeostasis, the  symmetric scores are normalized via a Softmax function with a temperature $\tau=0.01$, yielding the instantaneous adjacency matrix $\hat{\mathbf{A}}_{i j}^{(t)}$:
\begin{align}
   \hat{\mathbf{A}}_{i j}^{(t)}=\text{Drop}\left(\frac{\exp \left(\bar{e}_{i j}^{(t)} / \tau\right)}{\sum_{l=1}^{N} \exp \left(\bar{e}_{i l}^{(t)} / \tau\right)}\right),
    \label{eq15:deltaS}
\end{align}%
\noindent where $\text{Drop}(\cdot)$ is dropout layer with a drop probability $0.2$.

Incorporating the principles of STP, we update the synaptic connectivity $\mathbf{S}^{(t)}$ using $\hat{\mathbf{A}}^{(t)}$. This is implemented via a momentum-based update rule:
\begin{align}
   \mathbf{S}^{(t)}=\beta \ \mathbf{S}^{(t-1)}+(1-\beta) \hat{\mathbf{A}}^{(t)},
    \label{eq16:updateS}
\end{align}%
\noindent where $\beta \in [0,1]$ serves as a hyperparameter controlling the evolution intensity of the graph structure. Notably, setting $\beta=1$ reduces the system to a static graph, a property utilized in our ablation studies.

Finally, to mitigate noise and induce topological sparsity, we implement a dynamic Top-$k$ pruning strategy, retaining only the $k$ connections with the highest synaptic weights for each node. To address the non-differentiability of the Top-$k$ selection, we employ a gradient masking mechanism during training.  Specifically, the pruning operation acts as a dynamic mask: during backpropagation, gradients propagate back to the learnable parameters exclusively through the values of the retained active edges, while gradients for pruned connections are zeroed out. This ensures end-to-end optimization of the underlying attention weights based on the most task-relevant pathways. The operation is formulated as follows:
\begin{align}
   \hat{\mathbf{S}}_{i j}^{(t)}=\left\{\begin{array}{ll}
\mathbf{S}_{i j}^{(t)}, & \text { if } \mathbf{S}_{i j}^{(t)} \in \text{Top}_{k}\left(\mathbf{S}_{i,:}^{(t)}\right) \\
0, & \text { otherwise }
\end{array}\right.,
    \label{eq17:pruning}
 \end{align}%
\noindent where $\text{Top}_{k}(\cdot)$ extracts the indices of the $k$ largest elements in a row. 

\subsection{Graph Diffusion}
\label{subsec: GD}
Upon the inferred sparse topology $\hat{\mathbf{S}}$, we introduce the \textbf{GD} mechanism for signal propagation. Unlike traditional GCNs, we propagate information from the source node $v_0$ to other nodes via $M$ iterative diffusion steps. 
Crucially, we design the diffusion operator to achieve spatio-temporal decoupling. Since LIF neurons inherently encode temporal memory via membrane potential decay term $\tau_{\text{decay}}u^{(t-1)}_i$, we avoid the explicit addition of self-loops (i.e., the renormalization trick $\mathbf{A}+\mathbf{Id}$ as in Eq.~\ref{eq3:gcn}) to prevent redundant self-information accumulation. Furthermore, to rectify the asymmetry introduced by Top-$k$ pruning, we re-symmetrize the topology to ensure a valid undirected operator $\mathbf{P}^{(t)}$:
\begin{align}
   \mathbf{P}^{(t)}=\mathbf{D}^{-\frac{1}{2}} \left(\frac{\hat{\mathbf{S}}^{(t)}+(\hat{\mathbf{S}}^{(t)})^{\top}}{2}\right)\mathbf{D}^{-\frac{1}{2}},
    \label{eq19:operator}
\end{align}%
\noindent where $\mathbf{D}$ is the degree matrix of $(\hat{\mathbf{S}}^{(t)}+(\hat{\mathbf{S}}^{(t)})^\top)/2$.

To initiate the diffusion, we construct a global graph signal tensor $\mathbf{X}^{(t)}$. Consistently with our source node design, we assign $\mathbf{O}_0^{(t)}$ to the first row and zero-pad the remaining $N-1$ rows, reflecting the latent state of the network prior to spatial propagation. Formally:
\begin{align}
   \mathbf{X}^{(t)}=\left[\begin{array}{c}
\mathbf{O}_{0}^{(t)} \\
\mathbf{0}_{N-1}
\end{array}\right],
    \label{eq18:x}
\end{align}%
\noindent where $\mathbf{0}_{N-1}$ denotes the zero-initialized states of other non-source nodes.

Finally, the signal propagates through the graph via an $M$-step diffusion process:
\begin{align}
   \mathbf{I}_{i}^{(t)}=\left[\left(\mathbf{P}^{(t)}\right)^{M} \mathbf{X}^{(t)}\right]_{i}, \quad 1 \leq i<N.
    \label{eq20:diffusion}
\end{align}%
 Here, $M$ controls the diffusion depth, and the operator $[\cdot]_i$ extracts the feature row corresponding to node $v_i$.

The complete algorithm is provided in Appendix~\ref{app:algorithm}.

\begin{table*}[t]
\footnotesize
   \caption{Performance (\%) comparison with state-of-the-art methods on DVS-Gesture, CIFAR10-DVS, and N-Caltech101. \textbf{Bold} and \underline{underline} denote the best and second-best results, respectively. Results marked with $^*$ denote our reproduction.}
  \label{tab:dvs_result}
  \centering
  \begin{tabular}{llcccc}
    \toprule
    \textbf{Method} & \textbf{Architecture}  & $\mathbf{T}$ & \textbf{DVS-Gesture} & \textbf{CIFAR10-DVS}& \textbf{NCaltech101}\\  
    \midrule
     MLF \cite{feng2022multi} & VGG-9 & 5 & 85.77& 65.88 & 70.4 \\
     SSNN \cite{ding2024shrinking} & VGG-9 & 5 & 90.74 & 73.63 & 77.97 \\
     Spikingformer \cite{zhou2023spikingformer} &Spikingformer-2-256 & 5 & 93.40$^*$ & 77.50$^*$ & \underline{82.14$^*$} \\ 
     CogniSNN \cite{huang2025cognisnn} & WS-RGA-7 & 5 & \underline{95.41} & \underline{79.00} & 78.62 \\
     \midrule 
         \cellcolor{pink!25}\textbf{MorphSNN} &  \cellcolor{pink!25}DGD-SNN-7-3 & \cellcolor{pink!25}5&  \cellcolor{pink!25}\textbf{96.53$\hspace{0.1em} \scriptstyle \pm \hspace{0.1em} 0.11$}& \cellcolor{pink!25}\textbf{79.90$\hspace{0.1em} \scriptstyle \pm \hspace{0.1em} 0.13$} & \cellcolor{pink!25}\textbf{83.35$\hspace{0.1em} \scriptstyle \pm \hspace{0.1em} 0.18$} \\
     \midrule
     SEW-ResNet \cite{fang2021deep} & ConvNet & 16 & 97.92 & 74.40 & - \\
    Spikformer \cite{zhou2022spikformer} & Spikformer-2-256 & 16 & 98.30 & 80.90 & - \\
     Spikingformer \cite{zhou2023spikingformer} &Spikingformer-2-256 & 16 & 98.30 & 81.30 & - \\ 
     SpikingReformer \cite{shi2024spikingresformer}   & SpikingReformer-4-384 & 16& - & 78.80 & 81.29  \\
       Spike-D- Transformer \cite{zhou2024qkformer} & S-d-Transformer-2-256 & 16 & 98.26$^*$ & 80.00 & \underline{81.80}\\
      CogniSNN \cite{2025_huang_cognisnn} & ER-RGA-7 & 16 & \underline{98.61} & \underline{81.60$^*$} & 81.32$^*$
      \\
      \midrule
     \cellcolor{pink!25}\textbf{MorphSNN} &  \cellcolor{pink!25}DGD-SNN-7-3 & \cellcolor{pink!25}16& \cellcolor{pink!25}\textbf{98.96$\hspace{0.1em} \scriptstyle \pm \hspace{0.1em} 0.05$}& \cellcolor{pink!25}\textbf{83.70$\hspace{0.1em} \scriptstyle \pm \hspace{0.1em} 0.10$}& \cellcolor{pink!25}\textbf{82.74$\hspace{0.1em} \scriptstyle \pm \hspace{0.1em} 0.09$} \\
     
    \bottomrule
    \vspace{-8mm}
  \end{tabular}
\end{table*}

\subsection{Theoretical Analysis}
Graph diffusion in SNNs faces an intrinsic trade-off between feature binding and discrimination: excessive propagation leads to over-smoothing, causing the network to lose discriminative ability, while insufficient propagation results in under-activation, hindering feature synchronization and leading to the binding problem. Ideally, the network should converge to a global equilibrium, binding related features while preserving distinct class boundaries.

To rigorously quantify this balance, we model the signal propagation as a \textbf{Dirichlet energy minimization process}~\citep{shuman2013emerging}, serving as a metric for global signal smoothness on the learned topology $\hat{\mathbf{S}}^{(t)}$. Let $\mathbf{Y}$ denote the post-diffusion input matrix, where $\mathbf{y}_0$ is the source node input and $\mathbf{y}_{i}$ $(i>0)$ are the diffused states of non-source nodes. The Dirichlet energy of the diffused states $\mathbf{Y}$ is defined as:
\begin{align}
    \mathcal{E}_{\text {Dir }}(\mathbf{Y})=\operatorname{tr}\left(\mathbf{Y}^{\top} \mathbf{L}_{\text {norm }} \mathbf{Y}\right),
    \label{eq22:dirichlet}
\end{align}%
where $\mathbf{L}_{\text {norm }}=\mathbf{Id}-\mathbf{D}^{-1 / 2} \mathbf{S}_{\text {sym }}^{(t)} \mathbf{D}^{-1 / 2}$ is the symmetric normalized Laplacian, and $\mathbf{S}_{\text{sym}}^{(t)} = (\hat{\mathbf{S}}^{(t)}+(\hat{\mathbf{S}}^{(t)})^\top)/2$ is the symmetrized adjacency matrix. 
\begin{proposition}
Symmetry of $\mathbf{P}^{(t)}$ is sufficient for the diffusion to constitute a valid gradient flow on $\mathcal{E}_{\text {Dir }}$. This ensures real spectra, eliminating non-physical oscillations and guaranteeing stability of the gradient flow dynamics.
\end{proposition}

\begin{theorem}
\label{thm:gd_convergence}
The GD functions as a gradient descent step on $\mathcal{E}_{Dir}$ with the learning rate $\eta=0.5$. This represents the maximum stable step size, guaranteeing monotonic convergence to the energy equilibrium.
\end{theorem}
\textbf{Implication.} Theorem~\ref{thm:gd_convergence} guarantees convergence to a stable equilibrium. At this state, the signal is sufficiently smooth to bind related features yet retains necessary variance for classification.
Detailed proofs and empirical visualization are provided in Appendices~\ref{app:theoretical} and ~\ref{app:dirichlet}, respectively.

\section{Experiments}

\begin{table}
\centering
\caption{Comparison (\%) on static CIFAR-10/100 and ImageNet.}
\vspace{-2mm}
\label{tab:static}
\resizebox{\linewidth}{!}{%
\begin{tabular}[width=0.85\textwidth]{lcccccccc}
    \toprule
    \textbf{Method} & $\mathbf{T}$ &\textbf{CIFAR-10} &\textbf{CIFAR-100}& \textbf{ImageNet} \\
    \midrule

    ResNet-19 (ANN) & 1 & 94.97 & 75.35  & - \\
     Transformer (ANN) & 1 & \textbf{96.73} &\textbf{81.02}  & - \\
    \midrule
    \cellcolor{pink!25}\textbf{MorphANN (Ours)} &  \cellcolor{pink!25}1 & \cellcolor{pink!25}\underline{96.50} & \cellcolor{pink!25}\underline{80.29} & - \\
    \midrule
    SNASNet~\citep{kim2022neural} & 5 & 93.64 & 73.04 &- \\
    Spikformer~\cite{zhou2022spikformer}& 4 &  95.51 &78.21 & 73.38 \\
    Spikingformer~\citep{shi2024spikingresformer} & 4 &  \underline{95.81} &78.21 & \underline{74.79} \\
    S-Transformer~\citep{yao2023spike} & 4 & 95.60 & \underline{78.40} & 74.57\\  
    \midrule
    \cellcolor{pink!25}\textbf{MorphSNN (Ours)}  & \cellcolor{pink!25}4 &\cellcolor{pink!25}\textbf{95.97$\hspace{0.1em} \scriptstyle \pm \hspace{0.1em} 0.07$} &\cellcolor{pink!25}\textbf{79.87$\hspace{0.1em} \scriptstyle \pm \hspace{0.1em} 0.19$}  &\cellcolor{pink!25}\textbf{75.04}  \\
    \bottomrule
    \end{tabular}}
\end{table}

\textbf{Datasets and Implementation.}
 We evaluate MorphSNN on seven benchmarks: \textbf{DVS-Gesture}, \textbf{CIFAR10-DVS}, \textbf{N-Caltech101}, and \textbf{UCF101-DVS} for neuromorphic processing; \textbf{CIFAR-10/100} and \textbf{ImageNet} for static classification. Experiments on ImageNet are conducted on four NVIDIA H100 GPUs, while others datasets utilize a single RTX 4090. Further implementation details are provided in Appendix~\ref{app: experiment}.

\textbf{Comparative and Ablation Protocols.}
We benchmark MorphSNN against SOTA CNN and Transformer-based SNNs using three protocols:\ding{182} \textbf{Mechanism Isolation:} Comparing with random directed topology~\citep{2025_huang_cognisnn} to quantify GD and STSP gains; \ding{183} \textbf{Cross Domain Generalization:} Evaluating an ANN variant , MorphANN, with GeLU ($T=1$) to verify universality of GD; and \ding{184} \textbf{Open-World Adaptability:} Assessing OOD detection and perturbation robustness to examine how STSP-driven dynamic adaptation enhances resilience in volatile environments. Hyperparameter sensitivity, detailed visualizations, and energy efficiency analyses are provided in Appendices~\ref{app:sensitivity}, \ref{app:visualization}, and \ref{app:efficiency}.


\subsection{Overall Performance}

\textbf{Performance on Neuromorphic Benchmarks.}
As summarized in Table~\ref{tab:dvs_result},  MorphSNN excels in rapid-response inference. Under a 5-timestep constraint on DVS-Gesture, it outperforms CogniSNN by 1.12\%, demonstrating distinct advantages in rapid-response scenarios. In contrast, our reproduction of Spikingformer indicates that Transformer-based SNNs struggle in such settings, likely due to their reliance on long-term temporal accumulation. This advantage extends to CIFAR10-DVS and N-Caltech101, where MorphSNN surpasses the second-best methods by 0.9\% and 1.21\%, respectively.  When extended to 16 timesteps, MorphSNN maintains a substantial lead. On CIFAR10-DVS, it achieves a 2.1\% accuracy gain over CogniSNN, while preserving a slight 0.3\% edge on DVS-Gesture. Notably, on N-Caltech101, MorphSNN attains an accuracy of 82.74\%, surpassing the 81.80\% record held by the SOTA Spike-Driven Transformer.
 
\textbf{Performance on Static Benchmarks.}
As shown in Table~\ref{tab:static}, while the ANN variant, MorphANN, improves upon ResNet-19 under a comparable 10M parameter budget, it trails Vision Transformers.  However, in the spiking domain, MorphSNN exhibits superior efficacy. It surpasses Spikingformer by 0.16\% on CIFAR-10 and achieves a substantial 1.4\% margin over S-Transformer on the more complex CIFAR-100. Crucially, on ImageNet, MorphSNN achieves 75.04\% Top-1 accuracy, outperforming Spikingformer (74.79\%) and S-Transformer (74.57\%). This performance contrast underscores the intrinsic synergy between SNNs and undirected graph diffusion: spiking neurons effectively filter noise during membrane potential integration, whereas continuous-valued ANNs are prone to propagating and amplifying global errors across the graph.

Notably, the optimal performance on static datasets, as well as on CIFAR10-DVS and N-Caltech101, is achieved by the static MorphSNN variant ($\beta=1$). We provide a detailed analysis of this phenomenon in the next section.

\subsection{Ablation Studies}

\label{sec:ablation}
\begin{table}
\centering
\caption{Ablation studies (\%) . Sym denotes symmetric.}
\vspace{-2mm}
\label{tb:ablation}
\resizebox{\linewidth}{!}{%
\begin{tabular}[width=0.85\textwidth]{lccccccc}
    \toprule
    \textbf{Method} & \textbf{GD} & \textbf{STSP} & \textbf{Sym} & \textbf{DVS128} & \textbf{C10DVS} & \textbf{NCal101}   \\
    \midrule
    \textbf{Baseline} & \xmark &\xmark & \xmark &  92.01 &78.90 & 81.77 \\
 
    \textbf{Static}  & \cmark & \xmark & \cmark & \underline{95.44} & \textbf{79.90} & \textbf{83.35} \\

    \textbf{Asym} & \cmark & \cmark & \xmark & 94.44 & 78.20 & 82.14 \\
    \cellcolor{pink!25} \textbf{Full} &\cellcolor{pink!25} \cmark &\cellcolor{pink!25} \cmark &\cellcolor{pink!25} \cmark & \cellcolor{pink!25}\textbf{96.53} & \cellcolor{pink!25}\underline{79.40} &\cellcolor{pink!25} \underline{82.75}  \\
    \bottomrule
    \vspace{-8mm}
    \end{tabular}}
\end{table}
To quantify the individual contributions of the proposed GD and STSP, we construct four ablation variants: \ding{182} a \textbf{Baseline} employing random directed topology from CogniSNN~\citep{2025_huang_cognisnn}, which utilizes standard directed propagation instead of GD; \ding{183} a \textbf{Static} variant where fixing $\beta = 1$ freezes topological evolution, preserving the symmetric all-ones initialization; \ding{184} an \textbf{Asym} variant lacking adjacency matrix symmetrization; and \ding{185} the \textbf{Full} MorphSNN framework.

\begin{table}
\centering
\caption{Comparison and STSP ablation on UCF101-DVS.}
\vspace{-2mm}
\label{tab:ucf101-DVS}
\resizebox{\linewidth}{!}{%
\begin{tabular}[width=0.85\textwidth]{lcc}
    \toprule
    \textbf{Model} &  $\mathbf{T}$ & \textbf{Accuracy} (\%)  \\
    \midrule
    Event Frames~\citep{bi2020graph} & 8 & 57.9 \\  
    Res-SNN18 + RM~\citep{yao2023sparser} & 8 & 63.5 \\
     TIM~\citep{shen2024tim}& 10 &63.8\\
     \cmidrule{1-3}
    \cellcolor{pink!25}\textbf{MorphSNN (Static)}&\cellcolor{pink!25}10& \cellcolor{pink!25}\underline{65.81}\\
    \cellcolor{pink!25}\textbf{MorphSNN (Full)} &\cellcolor{pink!25}10& \cellcolor{pink!25}\textbf{67.12}\\
    \bottomrule
    \end{tabular}}
\end{table}

\begin{table*}[t]
\footnotesize
\centering
\caption{OOD detection scores obtained for \textbf{Full} MorphSNN in neuromorphic datasets. $\uparrow$ indicates higher is better, $\downarrow$ indicates lower is better. $^{\ast}$ denotes our proposed method. \textbf{Bold} and \underline{underline} denote the best and second-best results, respectively.}
\label{tb:OOD}
\begin{adjustbox}{width=\textwidth}
\begin{tabular}{ll ccccc ccccc ccccc}
    \toprule
    \multirow{2}{*}{\textbf{In dataset}} & \multirow{2}{*}{\textbf{OOD dataset}} & \multicolumn{5}{c}{\textbf{AUROC $\uparrow$}} & \multicolumn{5}{c}{\textbf{AUPR-Out $\uparrow$}} & \multicolumn{5}{c}{\textbf{FPR95 $\downarrow$}} \\
    \cmidrule(lr){3-7} \cmidrule(lr){8-12} \cmidrule(lr){13-17}
    & & \cellcolor{pink!25}\textbf{DGP}$^{\ast}$ & KNN & SCP & MSP & Energy & \cellcolor{pink!25}\textbf{DGP}$^{\ast}$ & KNN & SCP & MSP & Energy & \cellcolor{pink!25}\textbf{DGP}$^{\ast}$ & KNN & SCP & MSP & Energy \\
    \midrule
    \multirow{3}{*}{DVS128} 
    & C10-DVS 
    & \cellcolor{pink!25}\underline{99.22} & \textbf{99.25} & 87.19 & 78.01 & 83.32 
    & \cellcolor{pink!25}\underline{99.74} & \textbf{99.75} & 94.51 & 90.60 & 92.46 
    & \cellcolor{pink!25}\textbf{1.90} & \underline{2.70} & 8.40 & 38.70 & 25.50 \\
    & NCal-101 
    & \cellcolor{pink!25}\underline{99.48} & \textbf{99.71} & 89.48 & 75.57 & 62.37 
    & \cellcolor{pink!25}\underline{99.82} & \textbf{99.90} & 95.94 & 87.67 & 78.90 
    & \cellcolor{pink!25}\underline{0.55} & \textbf{0.33} & 5.91 & 42.01 & 70.46 \\
    & DVS-Lip 
    & \cellcolor{pink!25}\textbf{99.34} & \underline{99.28} & 76.04 & 79.48 & 87.20 
    & \cellcolor{pink!25}\textbf{99.95} & \textbf{99.95} & 97.32 & 97.89 & 98.63 
    & \cellcolor{pink!25}\textbf{0.56} & \underline{1.57} & 25.55 & 35.58 & 13.63 \\
    \bottomrule
\end{tabular}
\end{adjustbox}
\end{table*}

As detailed in Table~\ref{tb:ablation}, \textbf{Static} consistently outperforms \textbf{Baseline} across all benchmarks. The performance drop in \textbf{Baseline} stems from inefficient information routing within random topologies, validating that mathematically grounded diffusion rules are essential for effective graph utilization to prevent network degradation. Furthermore, the performance drop in \textbf{Asym} compared to \textbf{Full} empirically validates our theoretical proposition: enforced symmetry is a prerequisite for ensuring the convexity of Dirichlet energy minimization and convergence stability.

Crucially, we observe that the efficacy of STSP is highly sensitive to the temporal nature of the data. While the \textbf{Full} model outperforms the \textbf{Static} variant on DVS-Gesture, incorporating STSP leads to performance degradation on N-Caltech101 and CIFAR10-DVS.
We attribute this to the intrinsic nature of the data: STSP is designed to capture evolving temporal semantics. Since N-Caltech101 and CIFAR10-DVS are derived from static images via saccades, they lack genuine temporal dynamics. Consequently, STSP captures spurious correlations rather than meaningful structural evolution, introducing optimization instability.

To validate this hypothesis, we extend our evaluation to the large-scale dynamic dataset UCF101-DVS, as detailed in Table~\ref{tab:ucf101-DVS}. The results corroborate our conjecture: the \textbf{Full} variant regains its superiority and outperforms the \textbf{Static} version by 1.31\%. Moreover, both variants significantly exceed the second-best method, TIM. This confirms that STSP confers critical advantages specifically when processing inputs with rich, genuine temporal dynamics, whereas a static graph suffices for static-to-event datasets.

\subsection{Open-world Adaptability}
 \begin{figure*}[t]
    \centering
    \includegraphics[width=0.9\textwidth]{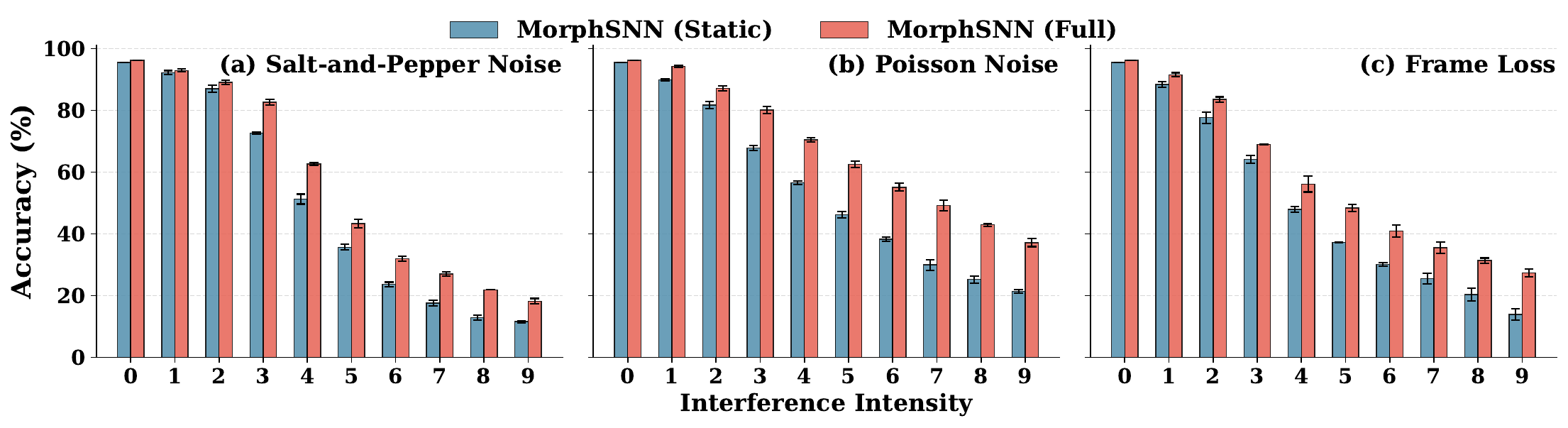}  
    \caption{Robustness evaluation of Static and Full models on DVS-Gesture under varying interference intensities.}
    \label{fig:robustness}
\end{figure*}

\subsubsection{OOD Detection}

To further validate the advantages of STSP in dynamic environments, we conduct OOD detection based on the \textbf{Structural Consistency Hypothesis}: ID samples activate topological patterns clustered around learned prototypes, whereas OOD inputs trigger divergent pathways. As Table~\ref{tb:OOD} shows, conventional methods (e.g., MSP, Energy) falter with AUROC scores of merely 75–85\% due to overconfident misclassifications, while SCP lacks the granularity for subtle semantic shifts.

To establish a performance upper bound, we introduce a non-parametric kNN baseline. While leveraging full geometric density achieves top-tier metrics, the $\mathcal{O}(N_{\text{train}})$ memory complexity of kNN is prohibitive for resource-constrained neuromorphic devices. In contrast, our Dynamic Graph Prototypes (DGP) method matches kNN’s precision with negligible overhead by compressing distributions into compact  $\mathcal{O}(N_{\text{class}})$ topological centroids.

On ID benchmarks, DGP consistently attains AUROC exceeding \textbf{99.2\%} and AUPR-Out over \textbf{99.7\%}. Crucially, DGP suppresses FPR95 to below \textbf{1.90\%}, notably dropping to \textbf{0.55\%} on N-Caltech101. Most notably, in the challenging Near-OOD DVS-Lip test, DGP achieves an FPR95 of \textbf{0.56\%}, significantly outperforming the 1.57\% upper bound set by kNN. It also maintains a distinct lead on CIFAR10-DVS (\textbf{1.90\%} vs. 2.70\%).

These results confirm that STSP-driven evolving topology serves as a robust, lightweight distribution descriptor capable of identifying even subtle semantic deviations. Biologically, this mechanism mimics stimulus-specific pathway recruitment, where the network learns to channel information through specialized synaptic routes. Consequently, ID samples follow these reinforced trajectories aligned with learned prototypes, while OOD inputs trigger unstructured activations that fail to match any established topological signature. Comprehensive descriptions of the DGP methodology are provided in Appendix~\ref{app:ood_detection}.

\subsubsection{Robustness Analysis}
To assess the robustness of the proposed STSP mechanism against real-world environmental volatility, we subject the model to comprehensive interference tests on DVS-Gesture. Specifically, we benchmark the \textbf{Full} MorphSNN against the \textbf{Static} baseline ($\beta=1$) under varying intensities of three perturbation protocols: Salt-and-Pepper Noise, Poisson Noise, and Frame Loss (details in the Appendix~\ref{app:robustness}).

As illustrated in Figure~\ref{fig:robustness}, while the performance of both models naturally declines with increasing perturbation intensity, the \textbf{Full} model exhibits significantly superior resilience compared to the \textbf{Static} version. Notably, under Salt-and-Pepper and Poisson noise, the \textbf{Full} model maintains a performance margin of approximately 10\% at high noise levels. We attribute this to the STSP mechanism: since transient noise spikes lack spatiotemporal consistency, they fail to accumulate sufficient trace values to consolidate into stable topological connections, thereby effectively blocking noise propagation.

Furthermore, under Frame Loss scenarios, the \textbf{Static} variant suffers from signal starvation due to zeroed-out inputs, leading to network silence. In contrast, the \textbf{Full} model exploits spatio-temporal compensation. Temporally, the hysteresis effect of synaptic traces preserves partial structural memory during input gaps; spatially, STSP adaptively redistributes weights to the remaining active nodes, ensuring continuous information flow. Consequently, MorphSNN effectively captures the robust underlying manifold features of neuromorphic streams, yielding enhanced interference resistance.

\section{Conclusion}

We present MorphSNN, a framework that transcends the limitations of static hierarchical SNNs by integrating \textbf{GD} and \textbf{STSP}. Theoretically, we prove that our diffusion-based propagation minimizes the Dirichlet energy, guaranteeing signal smoothness and stability. Empirically, MorphSNN achieves state-of-the-art performance across diverse neuromorphic and static classification benchmarks, demonstrating remarkable scalability to complex datasets like ImageNet. Crucially, the proposed STSP enables instance-specific topology evolution, exhibiting superior robustness against environmental interference compared to static baselines. Furthermore, we reveal that the evolved topologies encapsulate rich semantic information, facilitating near-perfect OOD detection without auxiliary training. Ultimately, MorphSNN suggests that neuromorphic intelligence emerges fundamentally from the fluidity of connectivity, advocating for a paradigm shift toward dynamic, self-organizing architectures in future brain-inspired computing.
\newpage
\section*{Impact Statement}
This paper presents work whose goal is to advance the field of Machine
Learning. There are many potential societal consequences of our work, none
which we feel must be specifically highlighted here.
\nocite{langley00}

\bibliography{example_paper}
\bibliographystyle{icml2026}

\newpage
\appendix
\onecolumn

\section{Pseudocode of DGD-SNN}
Due to space constraints, we present the algorithmic description of the DGD-SNN forward propagation here. Please refer to Algorithm~\ref{alg:algorithm}.
\label{app:algorithm}
\begin{algorithm}[tb]
    \caption{Forward Propagation of DGD-SNN}
    \label{alg:algorithm}
    \textbf{Input}: Input of source node $\mathbf{I}_{0}^{(t)}$, Initial Adjacency $\mathbf{S}^{(0)}$\\
    \textbf{Output}: Output $\mathbf{O}^{(t)}$ at timestep $t$
    \begin{algorithmic}[1] 
        \STATE Initialize $\mathbf{S}^{(0)}$ as an all-ones matrix
        \FOR{$ t = 1 $ to $T$}
        \item[] // \textit{\textbf{Phase 1: Calculate output of source node}}
        \STATE $\mathbf{O}_0^{(t)} \leftarrow AvgP(\text{ConvBNSN}_0(\mathbf{I}_0^{(t)}))$ in Eq.~\ref{eq5:source_node}
        \item[] // \textit{ \textbf{Phase 2: Update adjacency matrix using STSP}}
        \STATE Calculate hybrid state tensor $\tilde{\mathbf{O}}_i^{(t)}$ by Eq.~\ref{eq:hybrid_state}
        \STATE Calculate instantaneous attention adjacency probability matrix $\hat{\mathbf{A}}^{(t)}$ by combining Eqs.~\ref{eq10:project}--\ref{eq15:deltaS}
        \STATE Update the adjacency matrix: $\mathbf{S}^{(t)} \leftarrow\beta \cdot \mathbf{S}^{(t-1)}+(1-\beta) \cdot \hat{\mathbf{A}}^{(t)}$
        \STATE Execute pruning: $\hat{\mathbf{S}}^{(t)} \leftarrow \text{Top}_{k}(\mathbf{S}^{(t)})$ by Eq.~\ref{eq17:pruning} 
        \item[] // \textit{ \textbf{Phase 3: Execute GD}}
        \STATE Calculate the diffusion operation $\mathbf{P}^{(t)}$ by Eq.~\ref{eq19:operator}
        \STATE Calculate the inputs of others nodes $\mathbf{I}^{(t)}_{1\le i<N}$ by combining Eq.~\ref{eq18:x} and~\ref{eq20:diffusion}
        \item[] // \textit{ \textbf{Phase 4: Calculate total output}}
        \STATE Calculate the outputs of other nodes  $\mathbf{O}^{(t)}_{1\le i<N}$ by Eq.~\ref{eq:feature_extract_new}
        \STATE Aggregate the outputs from all nodes $\mathbf{O}^{(t)}_{0\le i<N}$ to obtain $\mathbf{O}^{(t)}$ by Eq.~\ref{eq9:total_output}
        \ENDFOR 
        \STATE \textbf{return} $\mathbf{O}^{(t)}$
    \end{algorithmic}
\end{algorithm}

\section{Theoretical Analysis and Proofs}
\label{app:theoretical}

\paragraph{Notation Clarification.} 
To maintain mathematical rigor and standard conventions within the spectral analysis, we clarify the following definitions and symbols used in this appendix:
\begin{itemize}
    \item \textbf{Strict Graph Definitions:} To ensure consistency with the gradient derivations, the degree matrix $\mathbf{D}$ is strictly defined based on the symmetrized adjacency matrix $\mathbf{S}_{sym}^{(t)} = (\hat{\mathbf{S}}^{(t)} + (\hat{\mathbf{S}}^{(t)})^\top)/2$. Specifically, $\mathbf{D}_{ii} = \sum_j (\mathbf{S}_{sym}^{(t)})_{ij}$. Consequently, the diffusion operator is $\mathbf{P} = \mathbf{D}^{-1/2}\mathbf{S}_{sym}^{(t)}\mathbf{D}^{-1/2}$, and the normalized Laplacian is $\mathbf{L}_{norm} = \mathbf{Id} - \mathbf{P}$.
    \item \textbf{Local Symbol Distinctions:} Certain symbols are defined \textbf{locally} for the spectral analysis and are distinct from global hyperparameters in the main text:
    \begin{itemize}
        \item The variable $t$ denotes the continuous physical diffusion time in the differential equation $\frac{d\mathbf{Y}}{dt} = -\mathbf{L}\mathbf{Y}$. It represents the extent of signal propagation depth. Theoretically, one discrete diffusion step (multiplication by $\mathbf{P}$) corresponds to advancing this physical time by a specific interval. This is distinct from the discrete inference timestep index used in the main methodology.
         \item The indices $k$ and $j$ denote the spectral mode or node indices, distinct from the Top-$k$ pruning parameter.
         \item The symbol $\lambda_k$ represents the $k$-th eigenvalue of the Laplacian, distinct from the synaptic trace decay factor $\lambda$ defined in Eq.~(\ref{eq11:trace}).
        \item The symbols $\alpha_k$ and $\beta_k$ represent the real and imaginary parts of eigenvalues, respectively, and are independent of the momentum coefficient $\beta$.
        \item The vector $\mathbf{y}$ (and matrix $\mathbf{Y}$) represents a general graph signal variable in the continuous state space, distinct from specific layer outputs.
    \end{itemize}
\end{itemize}

\subsection{Derivation of the Normalized Dirichlet Energy}
Before proving the propositions, we first clarify the derivation of the Dirichlet energy form presented in Eq.~(\ref{eq22:dirichlet}). The standard combinatorial Dirichlet energy measures the smoothness of absolute signal differences. However, in our diffusion framework, which involves degree-normalization steps ($\mathbf{D}^{-1/2}$), the appropriate physical quantity is the \textit{normalized} smoothness.

The normalized Dirichlet energy measures the weighted sum of squared differences between the degree-normalized signals of connected nodes:
\begin{align}
    \mathcal{E}_{Dir}(\mathbf{Y}) &= \frac{1}{2} \sum_{i,j} (\mathbf{S}_{sym}^{(t)})_{ij} \left\| \frac{\mathbf{y}_i}{\sqrt{d_i}} - \frac{\mathbf{y}_j}{\sqrt{d_j}} \right\|_2^2.
\end{align}
Expanding the quadratic term, we obtain:
\begin{align}
    \mathcal{E}_{Dir}(\mathbf{Y}) &= \frac{1}{2} \sum_{i,j} (\mathbf{S}_{sym}^{(t)})_{ij} \left( \frac{\mathbf{y}_i^\top \mathbf{y}_i}{d_i} - 2\frac{\mathbf{y}_i^\top \mathbf{y}_j}{\sqrt{d_i d_j}} + \frac{\mathbf{y}_j^\top \mathbf{y}_j}{d_j} \right) \nonumber \\
    &= \frac{1}{2}\sum_{i} \frac{\mathbf{y}_i^\top \mathbf{y}_i}{d_i} \underbrace{\sum_{j} (\mathbf{S}_{sym}^{(t)})_{ij}}_{d_i} - \sum_{i,j} \mathbf{y}_i^\top \frac{(\mathbf{S}_{sym}^{(t)})_{ij}}{\sqrt{d_i d_j}} \mathbf{y}_j + \frac{1}{2}\sum_{j} \frac{\mathbf{y}_j^\top \mathbf{y}_j}{d_j} \underbrace{\sum_{i} (\mathbf{S}_{sym}^{(t)})_{ij}}_{d_j} \nonumber \\
    &= \frac{1}{2}\sum_{i} \mathbf{y}_i^\top \mathbf{y}_i - \operatorname{tr}\left( \mathbf{Y}^\top (\mathbf{D}^{-1/2}\mathbf{S}_{sym}^{(t)}\mathbf{D}^{-1/2}) \mathbf{Y} \right) + \frac{1}{2}\sum_{j} \mathbf{y}_j^\top \mathbf{y}_j \nonumber \\
    &= \operatorname{tr}(\mathbf{Y}^\top \mathbf{Id} \mathbf{Y}) - \operatorname{tr}(\mathbf{Y}^\top \mathbf{P} \mathbf{Y}) \nonumber \\
    &= \operatorname{tr}\left( \mathbf{Y}^\top (\mathbf{Id} - \mathbf{P}) \mathbf{Y} \right) = \operatorname{tr}\left(\mathbf{Y}^\top \mathbf{L}_{norm} \mathbf{Y}\right).
\end{align}
This derivation confirms that minimizing the trace form in Eq.~(\ref{eq22:dirichlet}) is mathematically equivalent to enforcing local smoothness on the degree-normalized graph signals.

\subsection{Proofs}

\textbf{Proposition 1.} \textit{Symmetry of $\mathbf{P}^{(t)}$ is sufficient for the diffusion to constitute a valid gradient flow on $\mathcal{E}_{\text {Dir }}$. This ensures real spectra, eliminating non-physical oscillations and guaranteeing stability of the gradient flow dynamics.}

\textbf{Proof.}
The dynamics of the diffusion process are governed by the linear differential equation $\frac{d\mathbf{Y}}{dt} = -\mathbf{L}\mathbf{Y}$. Ideally, for this to represent a gradient flow on the energy landscape $\mathcal{E}_{Dir}(\mathbf{Y})$, the driving force $-\mathbf{L}\mathbf{Y}$ must align with the negative gradient direction $-\nabla_{\mathbf{Y}}\mathcal{E}_{Dir}$.
Since $\mathcal{E}_{Dir}$ is a quadratic form, its gradient is determined solely by the symmetric part of the Laplacian, $\mathbf{L}_{sym} = (\mathbf{L} + \mathbf{L}^\top)/2$.

\textit{Geometric Decomposition Analysis:}
Any square matrix $\mathbf{L}$ can be decomposed into a symmetric component $\mathbf{L}_{sym}$ and a skew-symmetric component $\mathbf{L}_{skew} = (\mathbf{L} - \mathbf{L}^\top)/2$. The diffusion dynamics can thus be split into two orthogonal fields:
\begin{equation}
    \frac{d\mathbf{Y}}{dt} = -\underbrace{\mathbf{L}_{sym}\mathbf{Y}}_{\text{Gradient Descent}} - \underbrace{\mathbf{L}_{skew}\mathbf{Y}}_{\text{Rotational Flow}}.
\end{equation}
The first term minimizes the energy. The second term, however, represents a conservative rotational flow that moves the state along the level sets of the energy function without reducing it (since $\mathbf{Y}^\top \mathbf{L}_{skew} \mathbf{Y} = 0$). If the graph is not symmetrized ($\mathbf{L}_{skew} \neq \mathbf{0}$), this rotational component introduces non-vanishing circulation in the phase space.

\textit{Spectral Interpretation:}
This geometric rotation manifests spectrally as complex eigenvalues $\lambda_k = \alpha_k + i\beta_k$ ($\beta_k \neq 0$). The time-evolution of the signal then becomes:
\begin{equation}
    \mathbf{y}_k(t) \propto e^{-\alpha_k t} (\cos(\beta_k t) - i \sin(\beta_k t)).
\end{equation}
The term $\sin(\beta_k t)$ creates non-physical oscillations. In SNNs, such oscillations are detrimental as they cause the membrane potential to fluctuate arbitrarily, potentially crossing the firing threshold $V_{th}$ due to topological asymmetry rather than feature accumulation. Therefore, enforcing symmetry ($\mathbf{L}_{skew} = \mathbf{0}$) is necessary to eliminate the rotational flow, ensuring $\beta_k = 0$ and guaranteeing that the diffusion is a pure gradient descent process with monotonic stability. \hfill $\square$

\textbf{Theorem 4.2.} \textit{The GD functions as a gradient descent step on $\mathcal{E}_{Dir}$ with the learning rate $\eta=0.5$. This represents the maximum stable step size, guaranteeing monotonic convergence to the energy equilibrium.}

\textbf{Proof.}
Using the trace form derived above and applying standard matrix calculus identities, the gradient of the energy is $\nabla_{\mathbf{Y}} \mathcal{E}_{Dir} = (\mathbf{L}_{norm} + \mathbf{L}_{norm}^\top)\mathbf{Y}$. Since $\mathbf{L}_{norm}$ is symmetric, this simplifies to:
\begin{equation}
    \nabla_{\mathbf{Y}} \mathcal{E}_{Dir} = 2\mathbf{L}_{norm}\mathbf{Y}.
\end{equation}
A discrete gradient descent step with learning rate $\eta$ is:
\begin{equation}
    \mathbf{Y}_{new} = \mathbf{Y}_{old} - \eta (2\mathbf{L}_{norm}\mathbf{Y}_{old}) = (\mathbf{Id} - 2\eta \mathbf{L}_{norm}) \mathbf{Y}_{old}.
\end{equation}
Substitute $\mathbf{L}_{norm} = \mathbf{Id} - \mathbf{P}$:
\begin{align}
    \mathbf{Y}_{new} &= (\mathbf{Id} - 2\eta (\mathbf{Id} - \mathbf{P})) \mathbf{Y}_{old} \nonumber \\
    &= ((1 - 2\eta)\mathbf{Id} + 2\eta\mathbf{P}) \mathbf{Y}_{old}.
\end{align}
Setting $\eta = 0.5$ eliminates the self-loop term $((1-1)\mathbf{Id})$, yielding $\mathbf{Y}_{new} = \mathbf{P}\mathbf{Y}_{old}$, which recovers our diffusion operator exactly.

\textit{Stability Boundary:}
The eigenvalues of $\mathbf{L}_{norm}$ are in $[0, 2]$. For the update $\mathbf{Y} \leftarrow (\mathbf{Id} - \mathbf{L}_{norm})\mathbf{Y}$ (i.e., $\eta=0.5$) to be non-divergent, the amplification factor $1 - \lambda(\mathbf{L}_{norm})$ must be in $[-1, 1]$. Since $\lambda(\mathbf{L}_{norm}) \in [0, 2]$, the factor lies in $[1-2, 1-0] = [-1, 1]$. Thus, $\eta=0.5$ is the largest possible learning rate that maintains stability. Specifically, at the bipartite limit $\lambda=2$, the factor is $-1$, corresponding to stable oscillation, whereas $\eta > 0.5$ would cause divergence. \hfill $\square$

\section{OOD Detection}
\label{app:ood_detection}
\subsection{Method}
We propose \textbf{Dynamic Graph Prototypes (DGP)}, an intrinsic OOD detection metric grounded in the \textit{Structural Consistency Hypothesis}: In-Distribution (ID) samples induce regular connectivity patterns via STSP, whereas OOD samples trigger disordered topologies.

To quantify this, we define the topological signature $\mathbf{z}$ of an input $\mathbf{x}$ by flattening the sequence of dynamic adjacency matrices evolved over $T$ timesteps:
\begin{align}
    \mathbf{z} = \Phi(\mathbf{x}) = \text{Flatten}\left( \left\{ \mathbf{S}^{(t)} \right\}_{t=1}^T \right) \in \mathbb{R}^{TN^2},
    \label{eq:feature_extraction}
\end{align}
where $\Phi$ serves as the feature extractor mapping the dynamic structure into a high-dimensional vector.

Post-training, we characterize the structural manifold of each ID class $n_{\text{class}}$ using the fixed backbone network. Given the ID training dataset $\mathcal{D}_{in} = \{(\mathbf{x}_i, y_i)\}_{i=1}^{N_{\text{train}}}$, where $y_i \in \{1, \dots, N_{\text{class}}\}$, we compute the prototypes $\boldsymbol{\mu}_{n_{\text{class}}}$ for each class, defined as the centroid of the structural features:
\begin{align}
    \boldsymbol{\mu}_{n_{\text{class}}} = \frac{1}{|\mathcal{D}_{in}^{n_{\text{class}}}|} \sum_{(\mathbf{x}, y) \in \mathcal{D}_{in}^{n_{\text{class}}}} \Phi(\mathbf{x}),
    \label{eq:prototype}
\end{align}
where $\mathcal{D}_{in}^{n_{\text{class}}}$ denotes the subset of training samples labeled as $n_{\text{class}}$. These prototypes serve as anchors for the known structural distributions.

During inference, to quantify the degree to which a test sample $\mathbf{x}^*$ is out-of-distribution, we measure the deviation of its structural pattern from the established ID manifolds. The anomaly score $S_{\text{DGP}}(\mathbf{x}^*)$ is defined as the Euclidean distance to the nearest graph prototype:
\begin{align}
    S_{\text{DGP}}(\mathbf{x}^*) = \min_{n_{\text{class}} \in \{1, \dots, N_{\text{class}}\}} \| \Phi(\mathbf{x}^*) - \boldsymbol{\mu}_{n_{\text{class}}} \|_2.
    \label{eq:anomaly_score}
\end{align}
A sample is classified as OOD if $S_{\text{DGP}} > \kappa$, where $\kappa$ is set to the 95th percentile of anomaly scores calculated on the ID training set, thereby maintaining a 95\% True Positive Rate (TPR) for known data.

\subsection{Baselines}
\label{app:ood_method}
\textbf{Maximum Softmax Probability (MSP).}  As a baseline method for OOD detection proposed by~\citet{hendrycks2016baseline}, MSP relies on the hypothesis that classifiers tend to assign higher prediction confidence to ID samples compared to OOD ones. The method utilizes the maximum probability from the softmax output layer directly as the anomaly detection score.

\textbf{Energy-based Score (Energy).} ~\citet{liu2020energy} proposed using Free Energy as a detection metric. Unlike softmax-based approaches that focus on relative probabilities, the energy score intrinsically reflects the log-likelihood relationship between the input sample and the probability density. Generally, ID samples exhibit lower energy levels, whereas OOD samples possess significantly higher energy.

\textbf{Deep KNN.} ~\citet{sun2022out} proposed a non-parametric approach that leverages the geometric information in the feature space. By computing the Euclidean distance between a test sample's embedding and its $k$-th nearest neighbor in the training set, Deep KNN utilizes feature space proximity as a direct uncertainty measure, effectively distinguishing OOD samples that lie far from the ID data manifold.

\textbf{Spike Count Pattern (SCP).} This serves as a feature-space ablation baseline. ~\citet{martinez2023novel} employed the same prototype distance metric mechanism as our proposed DGP method. However, it distinguishes itself by utilizing the mean firing rate of neurons in the penultimate layer as the feature vector. By contrasting SCP with DGP, we verify that the performance gains stem fundamentally from the topological evolution of the graph structure rather than merely from internal spiking statistical patterns.

\subsection{Dataset}
\textbf{DVS-Gesture.} Please refer to Appendix~\ref{app:dataset} for a detailed description. Serving as the ID dataset, we employ 288 samples from its training set as the ID samples.

\textbf{CIFAR10-DVS.} Please refer to Appendix~\ref{app:dataset} for a detailed description. In our OOD detection experiments, we utilize the entire test set of 1,000 samples as the OOD samples.

\textbf{NCaltech101.} Please refer to Appendix~\ref{app:dataset} for a detailed description. In our OOD detection experiments, we utilize the entire test set of 914 samples as the OOD samples.

\textbf{DVS-Lip.} We utilize the DVS-Lip dataset~\citep{tan2022multi} to establish a Near-OOD benchmark that rigorously tests robustness against semantic shifts within a consistent sensor domain. The dataset contains roughly 20,000 samples recorded with the same DVS128 sensor used for the in-distribution DVS Gesture task. This shared hardware setup forces the model to distinguish samples based on structural semantics rather than low-level sensor statistics or resolution artifacts. For the OOD evaluation, we employ the test set of 4,975 samples as the OOD data.

\subsection{Metrics}
\textbf{AUROC.} The Area Under the Receiver Operating Characteristic curve (AUROC) provides a threshold-independent assessment of the detector's overall discriminative capability. Defined as the area under the curve plotting True Positive Rate (TPR) against False Positive Rate (FPR), it intuitively represents the probability that the model assigns a higher anomaly score to a randomly selected OOD sample than to a random ID sample. A higher AUROC indicates better global separability between the two distributions.

\textbf{AUPR-Out.} The Area Under the Precision-Recall curve (AUPR-Out) serves as a complementary metric that evaluates detection performance by treating OOD samples as the positive class. Unlike AUROC, AUPR focuses specifically on the trade-off between detection precision and recall. This makes it a more rigorous indicator of reliability in scenarios with class imbalance, ensuring that the model achieves high detection accuracy without incurring a disproportionate number of false alarms.

\textbf{FPR95.} The False Positive Rate at $95\%$ True Positive Rate (FPR95) acts as a critical operational safety metric. It quantifies the percentage of ID samples that are wrongly classified as OOD when the decision threshold is calibrated to maintain a 95\% detection rate for OOD data. A lower FPR95 indicates the model's robustness in filtering out anomalies under strict constraints to guarantee high availability for normal inputs.

\section{Hyperparameter Sensitivity Analysis}
\label{app:sensitivity}

\begin{figure}[htbp]
    \centering
    \begin{subfigure}[b]{0.19\textwidth}
        \includegraphics[width=\textwidth]{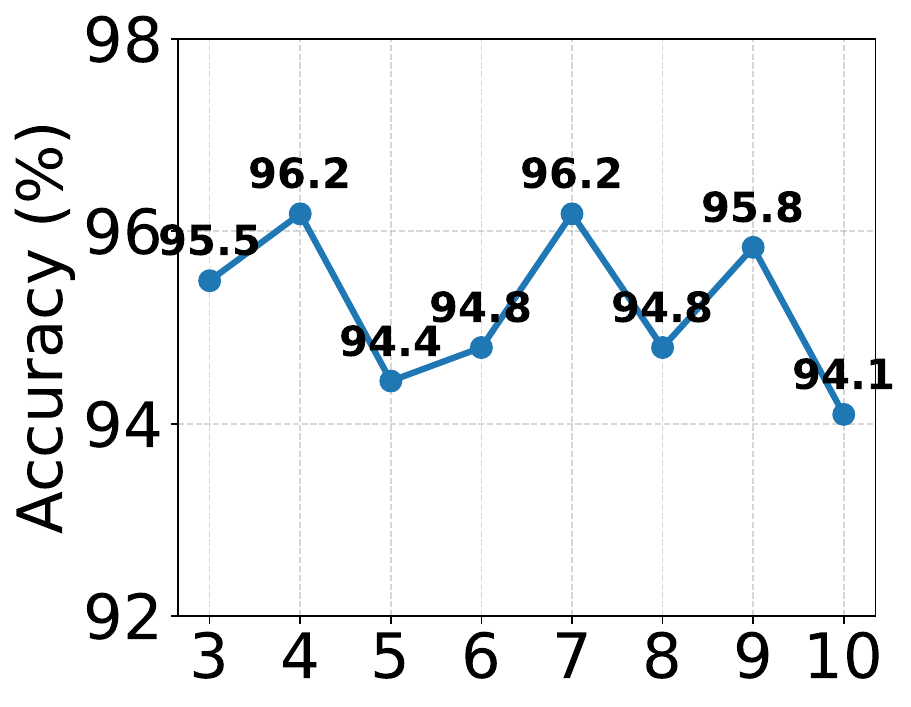}
        \caption{Number of Nodes $N$.} 
        \label{fig:N}
    \end{subfigure}
    \hfill
    \begin{subfigure}[b]{0.19\textwidth}
        \includegraphics[width=\textwidth]{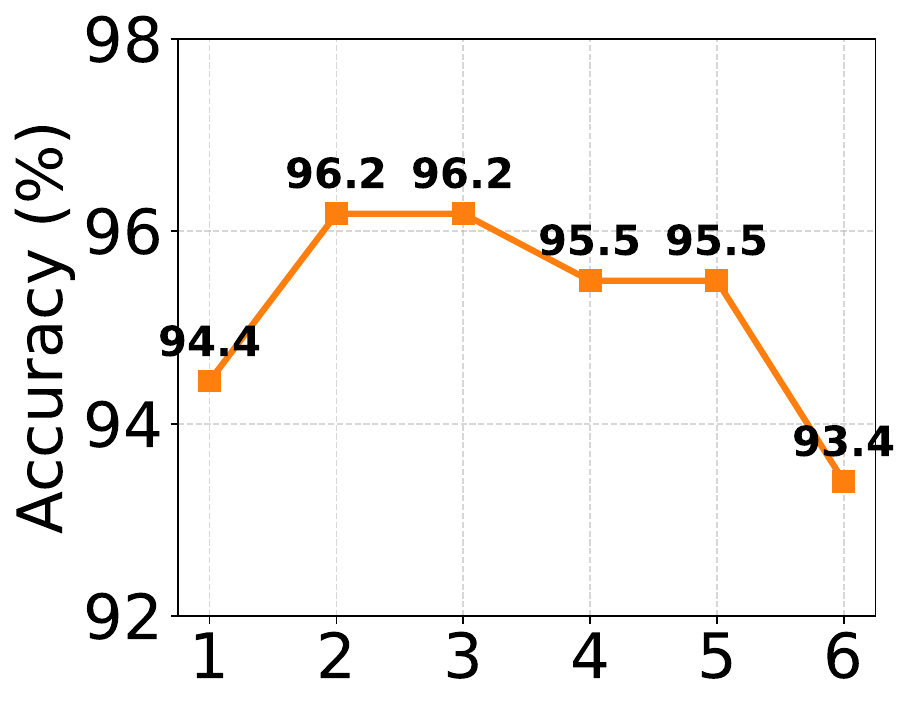}
        \caption{Diffusion Steps $M$.}
        \label{fig:M}
    \end{subfigure}
    \hfill
    \begin{subfigure}[b]{0.19\textwidth}
        \includegraphics[width=\textwidth]{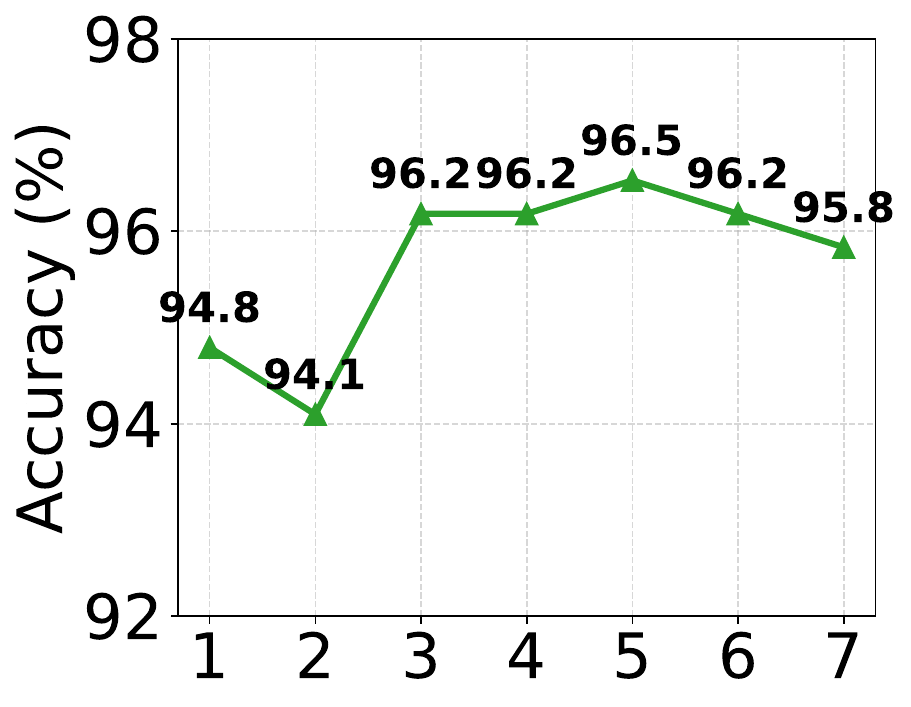}
        \caption{Top-$k$ Pruning $k$.}
        \label{fig:K}
    \end{subfigure}
    \hfill
    \begin{subfigure}[b]{0.19\textwidth}
        \includegraphics[width=\textwidth]{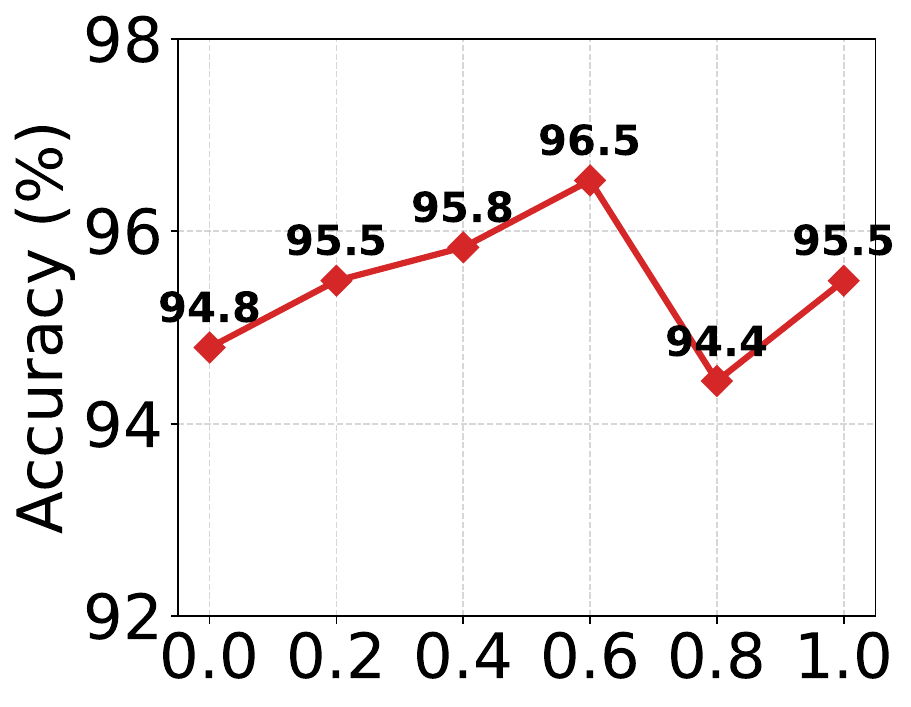}
        \caption{Trace Momentum $\lambda$.}
        \label{fig:lambda}
    \end{subfigure}
    \hfill
    \begin{subfigure}[b]{0.19\textwidth}
        \includegraphics[width=\textwidth]{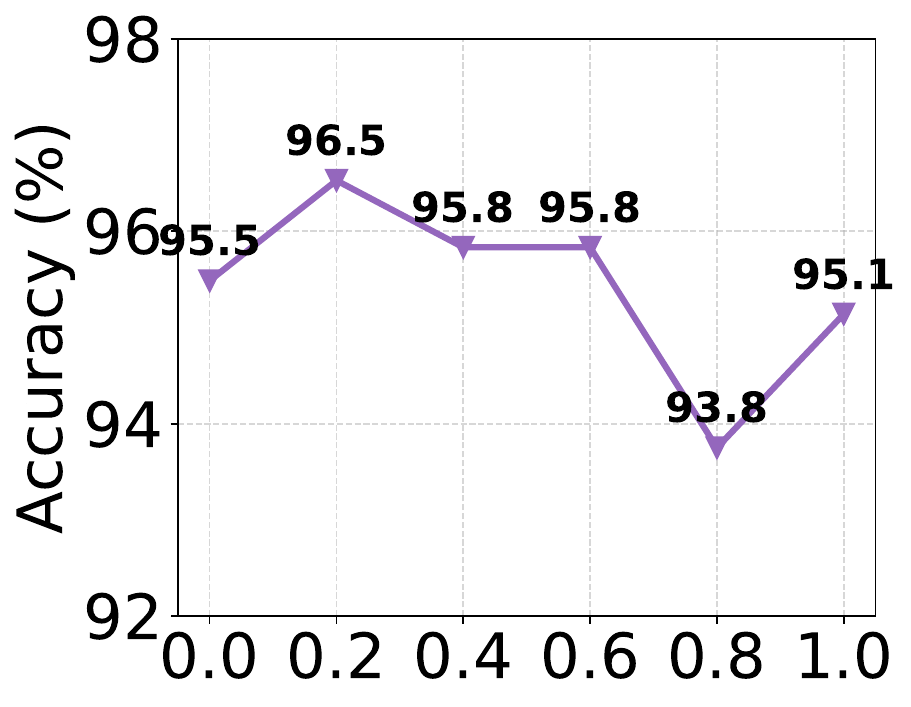}
        \caption{Update Momentum $\beta$.}
        \label{fig:beta}
    \end{subfigure}
    
    \caption{Hyperparameter sensitivity analysis on DVS-Gesture.}
    \label{fig:sensitivity}
\end{figure}



To investigate the impact of key configurations, we conduct a comprehensive sensitivity analysis on the DVS-Gesture under $T=5$. Unless otherwise specified, the default settings are: $N=7$, $M=2$, $K=3$, $\lambda=0.6$, and $\beta=0.2$.

\textbf{Analysis of $N$.} The number of nodes $N$ determines the model's capacity for feature abstraction and topological complexity. While small $N$ limits representational ability, excessive $N$ imposes computational overhead. As shown in Figure~\ref{fig:N}, the model achieves peak accuracy (96.2\%) at both $N=4$ and $N=7$. We select $N=7$ to allow for richer graph structural variations and to ensure consistency with configurations used for larger datasets.
 
\textbf{Analysis of $M$.} The diffusion step $M$ regulates the range of inter-node information exchange. Insufficient $M$ restricts signal propagation, whereas excessive $M$ leads to the over-smoothing problem, where node features become indistinguishable. Figure~\ref{fig:M} indicates optimal performance at $M=2$ and $M=3$. We adopt $M=2$ for computational efficiency. Notably, performance drops significantly at $M=6$, confirming that excessive diffusion dilutes the discriminative power of local features.

\textbf{Analysis of $k$.} This parameter controls the sparsity of the dynamic adjacency matrix. $k=7$ corresponds to a fully connected (non-pruned) graph, while $k=1$ forces a rigid tree-like structure. Results in Figure~\ref{fig:K} demonstrate that $k=5$ yields the highest performance. The degradation observed at $k=7$ validates the necessity of sparse pruning to remove noisy or redundant connections. Consequently, we fix $k=5$ for the subsequent analyses of momentum terms.

\textbf{Analysis of $\lambda$.} The synaptic trace momentum $\lambda$ governs the trade-off between historical memory and instantaneous sensory input during synaptic trace calculation. As illustrated in Fig.~\ref{fig:lambda}, we evaluate $\lambda$ across the range $\left \{ 0.0,0.2,0.4,0.6, 0.8,1.0\right \}$. The model achieves optimal performance, 96.5\% at $\lambda = 0.6$, Notably, $\lambda = 0.0$ and $\lambda =1.0$ correspond to reliance solely on instantaneous input and historical memory, respectively. The suboptimal performance observed at both boundaries verifies the necessity of effective spatio-temporal fusion.

 \textbf{Analysis of $\beta$.} The synaptic update momentum $\beta$ directly regulates the stability of the evolving graph topology. As shown in Fig.~\ref{fig:beta}, we test $\beta$ values within $\left \{ 0.0,0.2,0.4,0.6, 0.8,1.0\right \}$, observing peak performance at $\beta = 0.2$. $\beta = 0$ implies the graph connectivity is entirely determined by the instantaneous update, while $\beta =1 $ degenerates the network into a static topology. The performance degradation at both extremes confirms the critical role of controlled dynamic evolution in capturing time-varying features.

\section{Experiment Description and Dataset Pre-processing}
\label{app: experiment}

\subsection{Architecture Detail}
Table~\ref{tab:architecture} details the specific network configurations across diverse datasets. To accommodate varying task complexities, we design scalable MorphSNN variants ranging from lightweight to deep architectures.

\textbf{Neuromorphic Datasets.} For low-latency gesture recognition (DVS-Gesture), we employ an extremely lightweight 3-stage architecture with only \textbf{0.20 M} parameters, demonstrating the model's efficiency. For more challenging event streams (e.g., CIFAR10-DVS, NCaltech101, UCF101-DVS), we uniformly expand the channel width to 156 to capture fine-grained spatiotemporal patterns. 

\textbf{Static Image Datasets.} We adopt a pyramidal structure similar to standard CNNs. Notably, for the large-scale \textbf{ImageNet} dataset, the architecture is scaled up to 4 stages with \textbf{29.89 M} parameters and a $7\times7$ input stem, ensuring sufficient capacity for high-dimensional feature representation. 

\begin{table}[h]
\centering
\caption{Comparison of network architectures across all datasets}
\label{tab:architecture}
\renewcommand{\arraystretch}{1.4}

\newcommand{\highrow}{\rule[-2.5em]{0pt}{5.0em}}

\resizebox{\linewidth}{!}{
\begin{tabular}{c|c|c|c|c|c}
\hline
\multirow{3}{*}{\textbf{Stage}} & \multicolumn{3}{c|}{\textbf{Neuromorphic Datasets}} & \multicolumn{2}{c}{\textbf{Static Image Datasets}} \\ \cline{2-6} 
 & \textbf{DVS-Gesture} & \textbf{C10DVS / NCal101} & \textbf{UCF101-DVS} & \textbf{CIFAR-10 / 100} & \textbf{ImageNet} \\ \cline{2-6} 
 & (Input: $128\times128$) & (Input: $128\times128$) & (Input: $64\times64$) & (Input: $32\times32$) & (Input: $224\times224$) \\ \hline

Encoder & \{3x3, 32, AP\} & \{3x3, 156, AP\} & \{3x3, 156\} & \{3x3, 64\} & \{7x7, 64, S2, P3, AP\} \\ \hline

1 \highrow & 
$\left[ \begin{array}{c} 3\text{x}3, 32 \end{array} \right] \times 7$ & 
\multicolumn{2}{c|}{$\left[ \begin{array}{c} 3\text{x}3, 156 \end{array} \right] \times 7$} & 
$\left[ \begin{array}{c} 3\text{x}3, 128 \end{array} \right] \times 4$  & 
$\left[ \begin{array}{c} 3\text{x}3, 64 \end{array} \right] \times 8$ \\ \hline

2 \highrow & 
$\left[ \begin{array}{c} 3\text{x}3, 32 \end{array} \right] \times 7$ & 
\multicolumn{2}{c|}{$\left[ \begin{array}{c} 3\text{x}3, 156 \end{array} \right] \times 7$} & 
$\left[ \begin{array}{c} 3\text{x}3, 256 \end{array} \right] \times 4$  & 
$\left[ \begin{array}{c} 3\text{x}3, 128 \end{array} \right] \times 8$ \\ \hline

3 \highrow & 
$\left[ \begin{array}{c} 3\text{x}3, 32 \end{array} \right] \times 7$ & 
\multicolumn{2}{c|}{$\left[ \begin{array}{c} 3\text{x}3, 156 \end{array} \right] \times 7$} & 
$\left[ \begin{array}{c} 3\text{x}3, 512 \end{array} \right] \times 4$  & 
$\left[ \begin{array}{c} 3\text{x}3, 256 \end{array} \right] \times 16$ \\ \hline

4 \highrow & 
- & \multicolumn{2}{c|}{-} & 
- & 
$\left[ \begin{array}{c} 3\text{x}3, 512 \end{array} \right] \times 8$ \\ \hline

Feat. Size & 
$1\times1$ & \multicolumn{2}{c|}{$1\times1$} & $4\times4$ (then GAP) & $7\times7$ (then GAP) \\ \hline

FC & FC-11 & FC-10 / 101 &FC-101 & FC-10/100 & FC-1000 \\ \hline

Param (M) & 0.20  & \multicolumn{2}{c|}{4.72 } & 10.85  & 29.89  \\ \hline

\end{tabular}
}
\end{table}

\subsection{Training Configuration}
Table~\ref{tb:hyperparams} presents the comprehensive hyperparameter settings across all datasets. To adapt to the distinct data modalities, we employ differentiated optimization strategies. For static image datasets such as CIFAR and ImageNet, models are optimized using \textbf{AdamW} with a weight decay of 0.05 and larger batch sizes ranging from 64 to 256 under a uniform Cosine Annealing scheduler. In contrast, for neuromorphic datasets, we utilize \textbf{SGD} with zero weight decay to effectively preserve the sparsity of spiking activities. Due to the high memory overhead associated with multi-timestep processing, the batch size is restricted to 8, while learning rate schedules including Cosine or Step modes are empirically tailored to match the specific convergence profile of each event-driven task. To mitigate statistical variance, all experiments, with the exception of the large-scale ImageNet and UCF101-DVS datasets, are conducted across three independent runs, with the results reported as mean $\pm$ standard deviation.

\begin{table}[h]
    \centering
    \caption{Hyperparameter settings for different datasets.}
    \label{tb:hyperparams}
    
    \resizebox{0.9\linewidth}{!}{
        \begin{tabular}{lcccccc} 
            \toprule
            \textbf{Dataset} & \textbf{Optimizer} & \textbf{Weight Decay} & \textbf{Batch Size} & \textbf{Epochs} & \textbf{Init. LR} & \textbf{Schedule} \\
            \midrule
            
            CIFAR-10     & AdamW & 0.05 & 64 & 400 & $1 \times 10^{-3}$ & Cosine \\ 
            CIFAR-100    & AdamW & 0.05 & 64 & 400 & $1 \times 10^{-3}$ & Cosine \\
            ImageNet     & AdamW & 0.05 & 256 & 300 & $6 \times 10^{-4}$ & Cosine \\
            
            \midrule 
            
            DVS-Gesture  & SGD   & 0    & 8  & 300 & $1 \times 10^{-3}$ & Cosine \\
            CIFAR10-DVS  & SGD   & 0    & 8  & 192 & $1 \times 10^{-3}$ & Step \\
            N-Caltech101 & SGD   & 0    & 8  & 192 & $1 \times 10^{-3}$ & Step \\
            UCF101-DVS   & SGD   & 0    & 8  & 192 & $1 \times 10^{-3}$ & Step \\
            
            \bottomrule
        \end{tabular}
    }
\end{table}

\subsection{Datasets}
\label{app:dataset}

\textbf{CIFAR-10 $\&$ CIFAR-100.} CIFAR-10 consists of 60,000 color RGB images with a resolution of $32 \times 32$ pixels, distributed across 10 classes. Its more challenging counterpart, CIFAR-100, maintains the same total number of images and resolution but expands to 100 classes, resulting in a ten-fold reduction in sample size per class. Both datasets are split into 50,000 training samples and 10,000 test samples. We apply standard data augmentation and preprocessing protocols, consistent with the settings in the work~\cite{zhou2024qkformer}. For all static datasets, we employ a direct coding strategy, where the static image is repeated over 4 timesteps and fed directly into the first layer of the network.
 
\textbf{ImageNet-1k}~\citep{russakovsky2015imagenet}. As one of the most large-scale datasets in the field of image classification, ImageNet-1k serves as a primary benchmark for evaluating the final performance of neural network architectures. The dataset spans 1,000 diverse object categories, comprising approximately 1.28 million training images and 50,000 validation images. Following the experimental settings in~\cite{zhou2024qkformer}, we utilize the validation set for performance reporting and comparisons, and all input images are resized to the standard resolution of $224\times  224$

\textbf{DVS-Gesture}~\citep{2017_amir_low}. As an event-based action recognition benchmark, DVS-Gesture serves as the primary dataset for validating our proposed STSP mechanism. Captured by a $128 \times 128$ resolution DVS camera, it includes 11 hand gestures (e.g., waving, clapping) performed by 29 subjects under three distinct lighting conditions. Following the protocol established by~\citet{fang2021incorporating}, we partition the 1,464 samples into a training set of 1,176 samples and a test set of 288 samples. The event streams are fed into the network at their original $128\times128$ resolution without any additional data augmentation or preprocessing.

\textbf{CIFAR10-DVS}~\citep{2017_li_cifar10}. This dataset represents the neuromorphic conversion of CIFAR-10. It was generated by recording 10,000 static images from CIFAR-10 moving along repeated trajectories on a monitor using a DVS128 sensor. We maintain the original resolution of $128\times128$ for network input and split the dataset into training and validation sets with a 9:1 ratio. Standard SNN data augmentation techniques, including RandomCrop, RandomHorizontalFlip, and RandomRotation, are applied.

\textbf{NCaltech101}~\citep{2015_orchard_converting}. N-Caltech101 is the neuromorphic version of the static Caltech101 dataset, generated by mounting an event camera on a motorized pan-tilt unit to record saccadic movements across static images displayed on an LCD screen. It comprises 8,246 event stream samples covering 101 object categories and 1 background category, with an original resolution of $180\times240$. Following standard SNN practices, we discard the background class and downsample all samples to $128\times128$ using bilinear interpolation. The dataset split ratio and augmentation strategies are identical to those used for CIFAR10-DVS.
 
\textbf{UCF101-DVS}~\citep{bi2020graph}. Standing as one of the largest-scale action recognition benchmarks in SNNs, UCF101-DVS is a neuromorphic conversion of the origin RGB UCF101 dataset, generated by recording monitor displays via a DVS sensor. It comprises 13,320 event stream samples across 101 action categories with an original resolution of $240\times180$. We incorporate this benchmark specifically to highlight the efficacy of the STSP mechanism in processing complex dynamic datasets. To ensure a fair comparison, we adhere to the protocol established in TIM, spatial resolution was downsampled to $64\times64$, and the dataset was partitioned into training and validation sets with a 9:1 ratio, accompanied by standard data augmentation techniques.

\subsection{Robustness Evaluation Protocols}
\label{app:robustness}
To comprehensively evaluate the robustness of MorphSNN against environmental interference, we introduce three types of perturbations: Salt-and-Pepper Noise, Poisson Noise, and Frame Loss. The severity of each perturbation is controlled by an integer intensity level $\rho \in \left \{ 0,1,2,3,4,5,6,7,8,9\right \}$. The specific implementation details are defined as follows:

\textbf{Salt-and-Pepper Noise.} 
This noise is designed to simulate sensor defects such as dead pixels. For a given input event frame, we treat $\rho \times 0.02$ proportion of pixels as noise regions. Specifically, each pixel is randomly flipped to 0 (no event) or 1 (event) with probability $\rho \times 0.02$, while the remaining pixels remain unchanged.

\textbf{Poisson Noise.} 
The introduction of Poisson noise serves to simulate the randomness of photon arrival and shot noise inherent to the sensor. We inject noise conforming to a Poisson distribution into the input frame, where the intensity $\rho$ acts as the scaling factor for the Poisson rate parameter.

\textbf{Frame Loss.}
Frame drop experiments are used to evaluate a model's robustness against data transmission packet loss or sensor failures. We simulate this process by randomly discarding frames, specifically with each frame in the input sequence being randomly dropped with a probability of $\rho \times 0.05$.

\section{Visualization}
\label{app:visualization}

\subsection{Dirichlet Energy}
\label{app:dirichlet}
\begin{figure*}[t]
    \centering
    \includegraphics[width=0.8\textwidth]{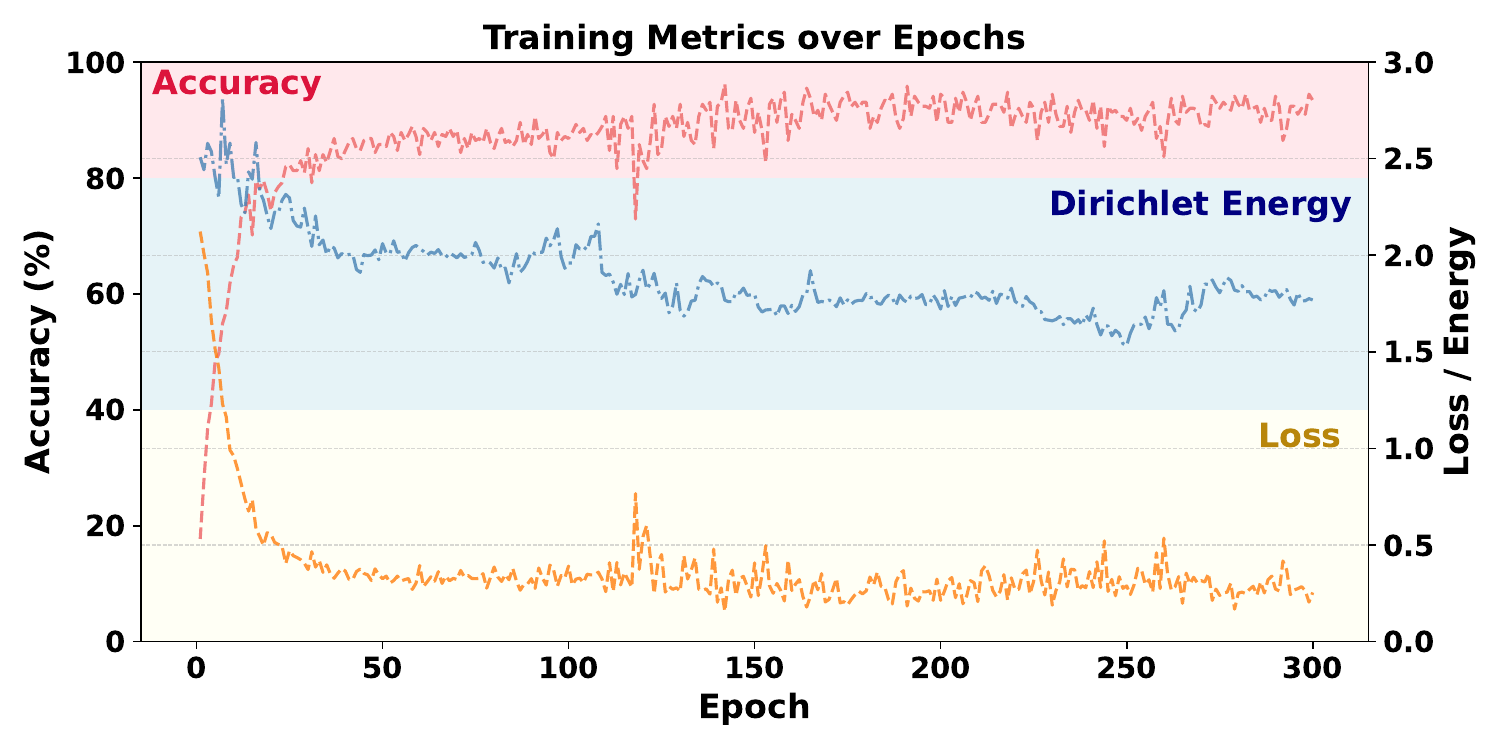}  
    \caption{Training Metrics over Epochs on DVS-Gesture.}
    \label{fig:metrics}
\end{figure*}
Figure~\ref{fig:metrics} presents the evolution of training metrics over 300 epochs on DVS-Gesture, specifically highlighting the trajectory of Dirichlet Energy (blue curve) alongside Accuracy (red) and Training Loss (orange). A key theoretical premise of our proposed GD mechanism is that it can be modeled as a gradient flow minimizing the Dirichlet energy of the graph signals, effectively smoothing the node features. Crucially, although we did not explicitly include a Dirichlet energy regularization term in the objective function, the experimental results show a consistent monotonic decrease in Dirichlet energy as the training progresses and the loss converges. The network naturally evolves towards a state of lower topological energy, validating that the diffusion process effectively functions as an implicit energy minimization operator, resulting in robust and noise-resilient feature representations.

\subsection{Feature Propagation Mechanisms} 
To investigate the internal dynamics of feature processing, we visualized the feature maps generated by Directed Connections (DC) and Graph Diffusion (GD), as shown in the Figure~\ref{fig:featuremap}. A distinct contrast is observed in their activation patterns. The DC mechanism (top row) exhibits highly localized and sparse activations, characterized by large dead zones (deep blue areas) and isolated high-response blocks. This suggests that traditional directed propagation is prone to feature degradation, where signals are progressively filtered or suppressed, leading to potential information loss and an over-reliance on narrow local features.

 In contrast, the GD mechanism (bottom row) demonstrates a more comprehensive and robust feature propagation capability. The heatmaps display a diffuse, granular distribution, indicating that the diffusion process effectively preserves information density across the spatial domain. Unlike DC, which suffers from signal decay in non-active regions, GD ensures that feature information is distributed globally. This mechanism allows the network to maintain a holistic representation of the input, effectively mitigating the feature degradation problem and ensuring that the read-out features remain rich and complete throughout the network layers.

\begin{figure*}[t]
    \centering
    \includegraphics[width=0.9\textwidth]{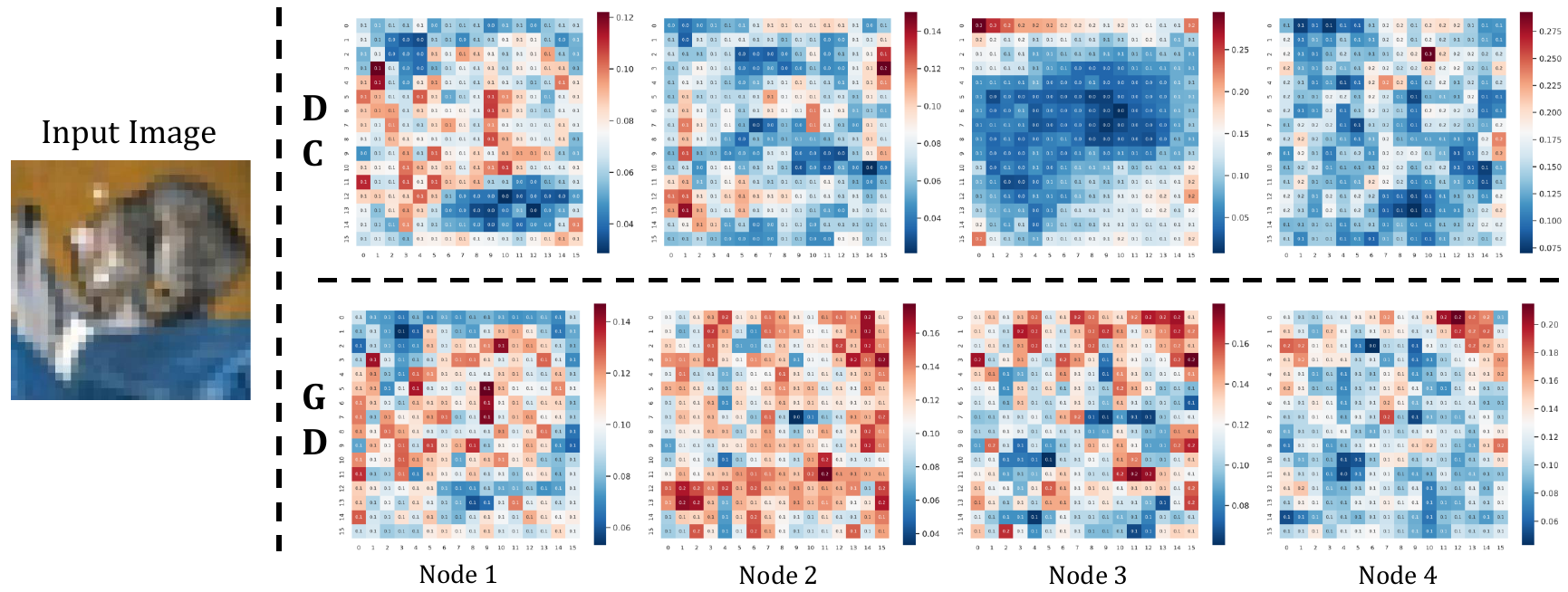}  
    \caption{Visualization of feature maps generated by directed connections (DC) and our graph diffusion (GD). }
    \label{fig:featuremap}
\end{figure*}

\subsection{Graph Structure Evolution}

\begin{figure*}[t]
    \centering
    \includegraphics[width=1.0\textwidth]{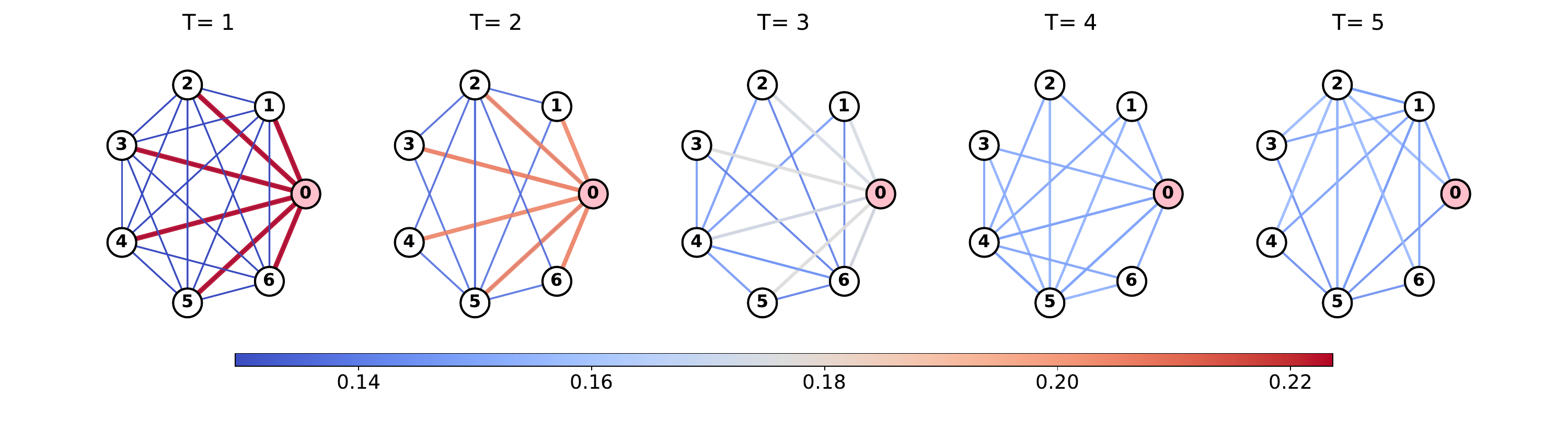}  
    \caption{Visualization of dynamic topological evolution for a single sample in DVS-Gesture.}
    \label{fig:graph}
\end{figure*}

As shown in the Figure~\ref{fig:graph}, this visualization tracks the connection changes of undirected graph on a DVS-Gesture sample over time ($T = 1$ to $5$), revealing that the STSP process exhibits distinct characteristics of Short-Term Synaptic Depression (STD). Initially, when $T=1$, the network establishes strong, dense connections (red) to rapidly capture the onset of neuromorphic signals. Subsequently, as the temporal information accumulates, the synaptic weights adaptively decay to a sparse state (blue). This STD-like behavior suggests that the network dynamically modulates its sensitivity: it prioritizes the capture of novel features in the initial phase while progressively suppressing temporal redundancy and noise in later stages, thereby self-organizing into an energy-efficient topology. Simultaneously, this distinct dynamic evolution trajectory offers sample-level topological cues for OOD detection, as anomalous inputs would typically induce deviant structural adaptation patterns.

\subsection{T-SNE in OOD} 
\begin{figure*}[t]
    \centering
    \includegraphics[width=0.9\textwidth]{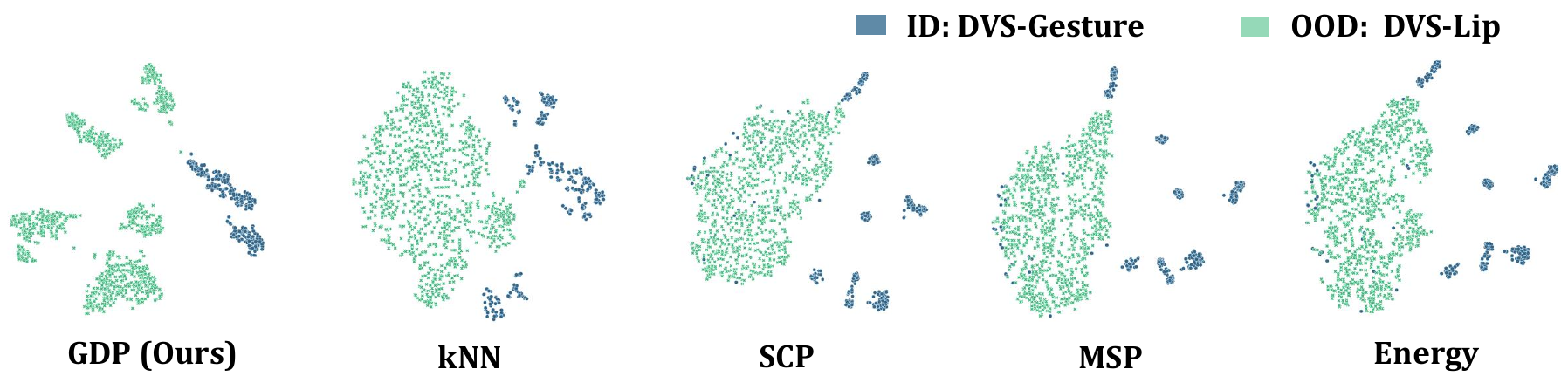}  
    \caption{t-SNE visualization of ID (DVS-Gesture) and OOD (DVS-Lip) distributions between DGP and other baselines.}
    \label{fig:tsne}
\end{figure*}
Figure~\ref{fig:tsne} visualizes the t-SNE embeddings of ID and OOD samples. A comparative analysis reveals fundamental differences in how the methods characterize latent data structures. As shown in the leftmost plot, our DGP method achieves the most distinct separation, exhibiting a unique `compact ID, clustered OOD' topology. The ID samples form a tightly condensed cluster, confirming that the STSP evolution converges to a consistent structural state for familiar inputs. Remarkably, the OOD samples are not mapped as uniform noise but are separated into distinct sub-clusters. This indicates that DGP preserves the intrinsic semantic manifolds of the anomalous data: distinct lip movements, despite being unseen, induce consistent and deterministic topological deviations. In stark contrast, baseline methods (kNN, SCP, MSP, Energy) display a `diffuse ID, entangled OOD' pattern. The OOD samples form a scattered, diffuse cloud that lacks internal structure and severely overlaps with the ID distribution. This chaotic entanglement stems from the over-confidence issue in static inference, where anomalous inputs trigger random high-confidence activations across ID classes. Consequently, these results confirm that the dynamic structural evolution captured by DGP provides a more robust and semantically meaningful fingerprint for anomaly detection than conventional static confidence scores.

\section{Analysis of Computation Efficiency}
\label{app:efficiency}
\subsection{Energy Analysis of GD and STSP}

Notably, in the spiking domain, the Static variant incorporates the Graph Diffusion mechanism compared to the Baseline, while the Full variant further integrates STSP. To demonstrate the hardware efficiency of these additional modules, we explicitly analyze their energy consumption based on 45nm CMOS process benchmarks ($E_{mac} = 4.6 \text{pJ}, E_{ac} = 0.9 \text{pJ}$, ~\cite{huang2024clif}).

Firstly, for the graph diffusion mechanism, given the adjacency matrix $\mathbf{S} \in \mathbb{R}^{N \times N}$ and the feature matrix $\mathbf{X} \in \mathbb{R}^{N\times(C\cdot H\cdot W)}$, we establish a rigorous theoretical upper bound on energy consumption by modeling the iterative process as dense matrix multiplication. Specifically, the operation $\mathbf{S}\cdot \mathbf{X}$ is decomposed into Multiplication operations (MULs) and Accumulation operations (ACs), formulated as follows:
\begin{align}
    \text{MULs}_{GD} &= T\cdot M \cdot (N^2) \cdot C \cdot H \cdot W, \label{eq:muls} \\
    \text{ACs}_{GD} &= T\cdot M \cdot (N(N-1)) \cdot C \cdot H \cdot W.\label{eq:acs}
\end{align}%
Substituting the parameters $N=7$, $T=5$, $M=2$, and dimensions $C=H=W=32$, and utilizing energy benchmarks from 45nm CMOS technology where $E_{mul} = E_{mac}-E_{ac} = 3.7 \text{pJ}$, we calculate the total energy consumption:
\begin{align} 
    E_{GD} = \text{MULs}_{GD} \cdot E_{mul} + \text{ACs}_{GD} \cdot E_{ac} \approx 0.0718 \text{ mJ}. \label{eq:gd_energy}
\end{align}
Even under this worst-case dense assumption, the total energy consumption is approximately 5.2\% of a standard convolutional layer (typically $>$ 1.39 mJ~\citep{horowitz20141}). Crucially, in the actual MorphSNN implementation, the expensive floating-point multiplications are largely bypassed due to the binary nature of the spiking input $\mathbf{X}$, the initialization sparsity where only the first row is active, and the topological sparsity of the pruned $\mathbf{S}$. Consequently, the real-world energy footprint will be significantly lower than this theoretical upper bound.

Secondly, for the STSP mechanism, the energy consumption is primarily attributed to Average Pooling, Feature Projection, Trace Update, Attention Calculation, and Structure Update. The derivation of the computational complexity for the principal components of STSP over $T$ timesteps is as follows:
\begin{align}
    \mathrm{ACs}_{\mathrm{stsp}} &\approx T \cdot[\underbrace{C H W}_{\text{Pooling}}+\underbrace{ 2N C(C-1)}_{\text{Projection}}+\underbrace{N^{2}(C-1)}_{\text{Attention}}] \approx 1.93 \times 10^{5},\\
    \mathrm{MACs}_{\mathrm{stsp}} &\approx T \cdot[ \underbrace{N \cdot C^{2}}_{\text{Projection}} +\underbrace{2 N C}_{\text{Trace}}+\underbrace{\left(2 N C^{2}+N^{2} C\right)}_{\text{Attention}}+\underbrace{2 N^{2}}_{\text{S-Up}}] \approx 9.45 \times 10^{4}.
    \label{eq:MACs_stsp}
\end{align}
By substituting into the energy consumption formula $E = \text{MACs} \cdot E_{mac} + \text{ACs} \cdot E_{ac}$, we calculate the total energy consumption of the STSP mechanism to be merely 0.00062 mJ. In conclusion, the combined energy overhead of GD and the STSP mechanism is negligible compared to traditional convolutional operations. This efficiency stems from the fact that graph diffusion is conducted via highly sparse matrix computations, while STSP operates solely on channel-wise firing rate representations with collapsed spatial dimensions, rather than within the high-dimensional pixel space.

\subsection{Analysis of Firing Rates}
\begin{figure*}[t]
    \centering
    \begin{subfigure}[b]{0.48\textwidth}
        \centering
        \includegraphics[width=\linewidth]{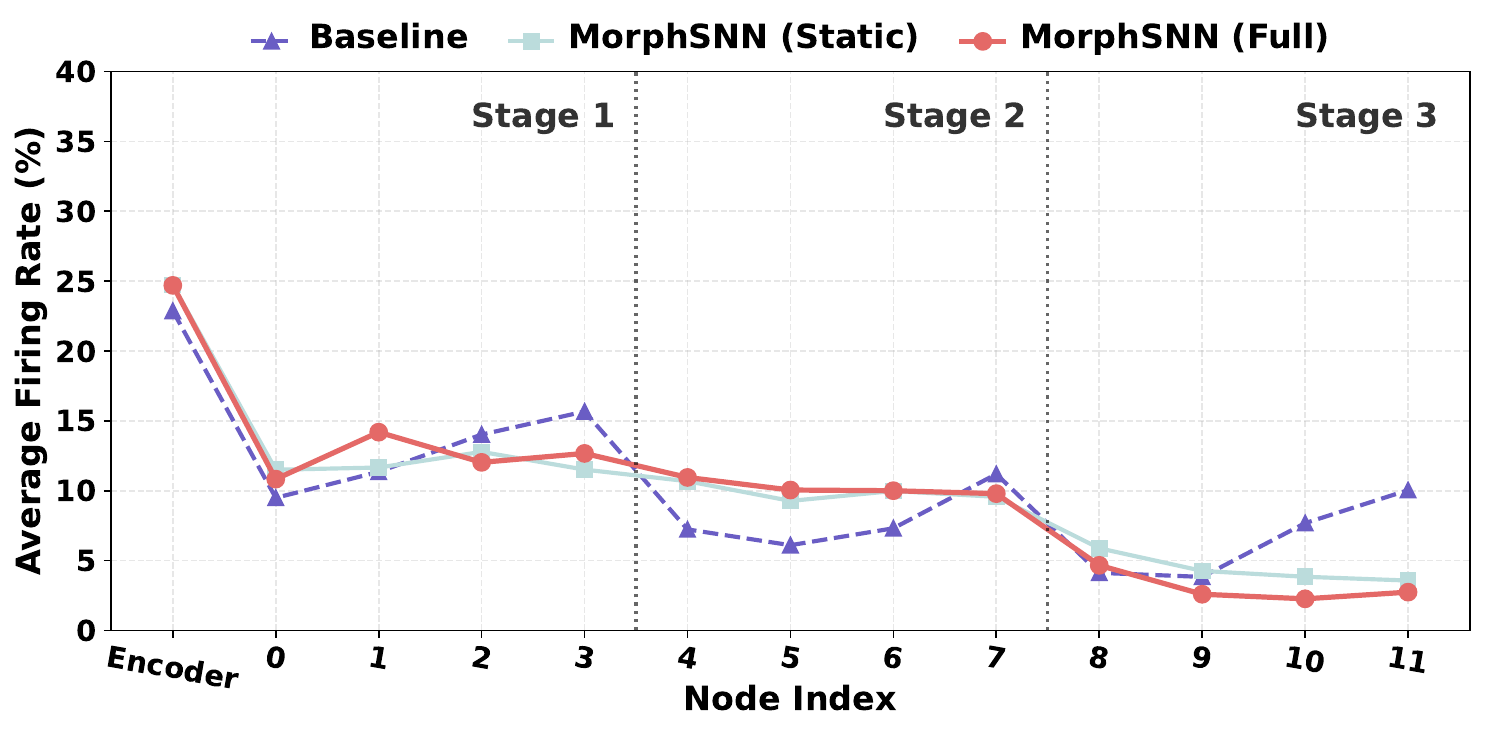}
        \caption{Results on CIFAR10.}
        \label{fig:cifar}
    \end{subfigure}
    \hfill 
    \begin{subfigure}[b]{0.48\textwidth}
        \centering
        \includegraphics[width=\linewidth]{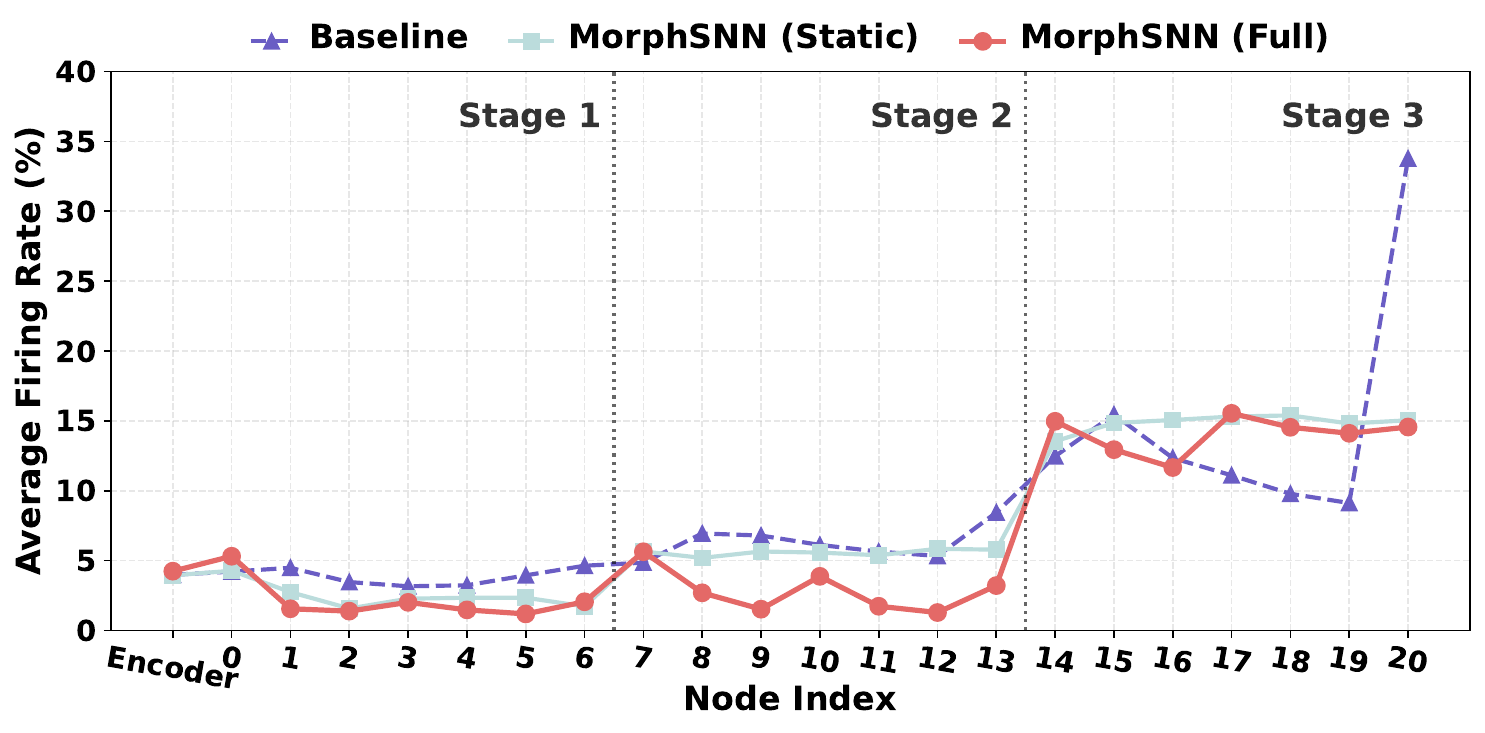}
        \caption{Results on DVS-Gesture.}
        \label{fig:gesture}
    \end{subfigure}
    
    \caption{Averaging Firing Rate of each node.}
    \label{fig:firingRate_combined}
\end{figure*}

Firing rate serves as a fundamental metric in neuromorphic computing, acting as a quantitative proxy for both neural excitability and information density. In biological neural networks, while high-frequency firing typically signifies strong responses to specific stimuli, biological systems globally optimize for metabolic efficiency through sparse coding mechanisms. For SNNs, this metric is critical as it correlates directly with the dynamic power consumption of neuromorphic hardware. Consequently, an ideal SNN model should achieve high task performance with minimal spiking activity, effectively balancing representational capacity with energy efficiency.

As illustrated in Figure~\ref{fig:firingRate_combined}, we evaluate the average firing rates of all test samples across individual nodes for the Baseline, Static, and Full variants discussed in Section~\ref{sec:ablation} on the DVS-Gesture dataset under $T=5$. Our node-wise analysis highlights significant differences in signal propagation dynamics. Specifically, the Baseline, utilizing directed graph connections, exhibits uncontrolled activity surges in deeper layers, reaching a firing rate of 33.78\%. This phenomenon is characteristic of ResNet-based SNNs, where excessive information accumulation along residual branches leads to over-firing. In contrast, both the Static and Full versions of MorphSNN maintain a stable and sparse firing regime. Notably, only the source nodes of the three stages (Nodes 0, 7, and 14) exhibit moderately higher firing rates, acting as information hubs, while subsequent nodes remain sparse. This ``low firing rate, high accuracy'' outcome demonstrates the signal propagation efficiency of our undirected graph diffusion and the robust representational capability of individual neural nodes. Furthermore, compared to the Static version, the Full version, following structural evolution via STSP, achieves an even lower average firing rate, demonstrating extreme efficiency in parameter utilization.

\subsection{Actual computation cost}

Following the detailed computation overhead analysis of GD and STSP, as well as the node-wise firing rates, we evaluate the actual energy consumption of the four variants introduced in Section~\ref{sec:ablation} (MorphANN, Baseline, Static, and Full) on the CIFAR-10 dataset, adopting the methodology from~\citet{sun2025ilif}.

Specifically, the computation overhead of neural networks is primarily assessed based on the Total Synaptic Operations (SOPs) during inference, which generally consist of MAC and AC operations. 
In general, ACs are defined as: 
\begin{equation}
    \text{ACs} = \sum_{t=1}^{T} \sum_{l=1}^{L-1} \sum_{i=1}^{N^l} f_i^l \, s_i^l[t],
\end{equation}
where fan-out $f_i^l$ represents the number of outgoing connections, $T$ denotes the simulation length in timesteps, $L$ is the total number of layers, and $N^l$ is the number of neurons in the $l$-th layer. Similarly, MACs are defined as: 
\begin{equation}
    \text{MACs} = \sum_{l=1}^{L-1} \sum_{i=1}^{N^l} f_i^l.
\end{equation}

In ANNs, given a specific architecture, SOPs are predominantly composed of MACs and remain constant. Conversely, in event-driven SNNs, SOPs are dictated by the firing rates of each layer and consist mainly of sparse ACs.

\begin{table}[h]
    \centering
    \caption{Computation Cost of different models on CIFAR-10.}
    \label{tb:energy}
    \resizebox{0.9\textwidth}{!}{
        \begin{tabular}{lcccccccc}
            \toprule
            \multirow{2}{*}{\textbf{Method}} & \multirow{2}{*}{\textbf{T}} & \textbf{Time} & \textbf{Param} & \textbf{Fire Rate} & \textbf{ACs} & \textbf{MACs} & \textbf{SOP Energy} & \textbf{Accuracy} \\
             & & \textbf{(s)} & \textbf{(M)} & \textbf{(\%)} & \textbf{(M)} & \textbf{(M)} & \textbf{($\mathbf{\mu}$J)} & \textbf{(\%)} \\
            \midrule
            \textbf{MorphANN} & 1 & 32  & 11.03 & -     & 0.29  & 400.20 & 1841.18 & 96.50 \\
            \textbf{Baseline} & 4 & 70  & 10.64 & 10.07 & 68.24 & 1.77   & 69.558  & 92.01 \\
            \textbf{Static}   & 4 & 72  & 10.85 & 9.91  & 69.69 & 1.77   & 70.863  & 95.97 \\
            \textbf{Full}     & 4 & 121 & 11.54 & 9.86  & 68.10 & 3.69   & 78.264  & 95.54 \\
            \bottomrule
        \end{tabular}
    }
\end{table}

Table~\ref{tb:energy} presents the specific computational cost metrics for the four variants. Here, T=1 indicates the ANN model, and `Time (s)' denotes the actual running time (in seconds) required to train one epoch on an NVIDIA RTX 4090 GPU. The results demonstrate that although MorphANN achieves the highest training accuracy, it consumes approximately 26$\times$ more energy compared to the three SNN variants. Among the SNN models, the Baseline, with its random connectivity, yields the lowest energy consumption but suffers from slightly degraded accuracy. The Full model incorporates the STSP mechanism; however, this proves somewhat superfluous for static datasets like CIFAR-10, incurring additional energy overhead. In contrast, the Static model, serving as a paradigm tailored for static datasets, effectively strikes a balance between accuracy and efficiency, achieving performance comparable to the ANN while maintaining significantly lower power consumption. It is worth noting that in terms of actual training time, Full consumes nearly twice as much time as Static. This does not imply that STSP incurs extremely high computational overhead, but rather stems from the fact that STSP operates on the intrinsic timestep, necessitating a single-step serial architecture rather than multi-step parallel processing. This is a challenge we must overcome in the future.
\begin{table}[h]
    \centering
    \caption{Detailed energy consumption breakdown for different methods on CIFAR-10.}
    \label{tb:energy_breakdown}
    \resizebox{0.98\textwidth}{!}{%
        \begin{tabular}{lcccccc}
            \toprule
            \multirow{2}{*}{\textbf{Method}} & \multicolumn{3}{c}{\textbf{Mem. Read \& Write}} & \multirow{2}{*}{\shortstack{\textbf{Synaptic \&} \\ \textbf{Neuron Op. (mJ)}}} & \multirow{2}{*}{\textbf{Addr. ($\mathbf{\mu}$J)}} & \multirow{2}{*}{\textbf{Total (mJ)}} \\
            \cmidrule(lr){2-4}
             & \shortstack{\textbf{Membrane} \\ \textbf{Potential (mJ)}} & \textbf{Parameters (mJ)} & \textbf{In / Out (mJ)} & & & \\
            \midrule
            \textbf{MorphANN} & 0 & 39.8431 & 39.8192 & 1.8412 & 0.5563 & \textbf{81.5041} \\
            \textbf{Baseline} & 13.2920 & 6.6460 & 0.0004 & 0.0696 & 59.9272 & \textbf{20.0676} \\
            \textbf{Static}   & 13.3569 & 6.6785 & 0.0005 & 0.0709 & 60.2313 & \textbf{20.1666} \\
            \textbf{Full}     & 13.5729 & 6.7214 & 0.0074 & 0.0691 & 60.3827 & \textbf{20.4313} \\
            \bottomrule
        \end{tabular}%
    }
\end{table}

Furthermore, according to~\citep{lemaire2022analytical}, in actual hardware execution, memory read/write operations (particularly for membrane potentials and parameters) are the dominant factors in energy consumption, dwarfing the negligible energy cost of synaptic operations. As shown in Table~\ref{tb:energy_breakdown}, energy consumption in ANNs is dominated by parameter retrieval and input/output (In/Out) operations. In SNNs, the irregular memory access patterns associated with membrane potential updates and sparse connectivity contribute significantly to the total cost. Consequently, the ranking of actual energy consumption aligns with intuitive expectations: MorphANN $\gg$ Full $\gtrsim$ Static $\gtrsim$ Baseline.

\section{Discussion}
\textbf{Prospect.}  As a mainstream paradigm in neuromorphic computing, SNN research benefits from placing greater emphasis on brain-inspired architectural exploration rather than focusing exclusively on performance benchmarks. Given that ANNs already excel at static classification, the unique promise of SNNs lies in mimicking biological operational mechanisms to bridge the gap in processing dynamic streaming scenarios. MorphSNN addresses this by incorporating bio-inspired undirected diffusion and structural plasticity into signal propagation. Looking ahead, we aim to pursue more granular innovations in brain-inspired architectures, extending beyond conventional convolution-based node designs. We aspire to endow neuromorphic models with the vitality needed for complex, open-ended environments, such as continual learning and spatial navigation. We believe these challenging domains deserve increased attention, broadening the scope of SNN research beyond single classification tasks.

\textbf{Reproducibility.} To ensure the reproducibility of our results, all experiments are implemented using PyTorch, SpikingJelly, and Timm, with the random seed strictly fixed at 2020. We have provided comprehensive descriptions of all architectural details. Our source code is publicly available on~ \url{anonymous.4open.science/r/MorphSNN-B0BC}.

\end{document}